\documentclass{article}

\usepackage{microtype}
\usepackage{graphicx}
\usepackage{subfigure}
\usepackage{booktabs} 

\usepackage{hyperref}

\usepackage[accepted]{icml2018}

\usepackage{amsthm}
\usepackage{amsmath}
\usepackage{amssymb}
\usepackage{bm}
\usepackage{bbm}
\usepackage{bbold}
\usepackage{mathtools}
\usepackage{enumitem}
\usepackage{breqn}
\usepackage{titletoc}
\usepackage{macros}
\usepackage{multirow}

\def\mmax{SA\log_2 \( \frac{8T}{SA} \)}

\newcommand{\MainResult}[3]{
#1
If \ubevs is run on an $H$-horizon MDP with $S$ states and $A$ actions where the successor states $s'$ is sampled from a fixed distribution $\mu$ then with probability at least $1-\delta$ the regret is bounded by the minimum between:
\begin{equation}
#2
 \underbrace{\tilde O \(\sqrt{SAT} + \frac{S^2AH^2\sqrt{H}}{\sqrt{\mu_{min}}} + \frac{SAH^2}{\mu_{min}}\)}_{\text{CMAB Analysis}}
\end{equation}
and
\begin{equation}
#3
 \underbrace{\tilde O \( H\sqrt{SAT} + S^2AH^{2} + S\sqrt{S}AH^3\)}_{\text{MDP Analysis}}
\end{equation}
jointly for all timesteps $T$.
}

\newcommand{\propbandit}[1]{
\begin{proposition}
#1
Assume that in every state there exists an action $a^*(s)$ that achieves the maximum instantaneous reward $r^* = \max_{s,a}r(s,a)$, i.e., $r^* = \bar r(s,a^*(s)) \; \forall s$ and that the MDP has a finite maximum mean hitting time $T_{hit}$ under any policy. For any initial state $(s)$ and any $T \geq 1$, with probability at least $1-\delta-o(\delta)$ the regret of \ucrl on such MDP is bounded by: 
\begin{equation*}
\tilde O\( \sqrt{SAT} + \sqrt{ST_{hit}}(SA)^2 T_{hit}^{M^{\pi^*}} + DSA\).
\end{equation*}
\end{proposition}
}

\newcommand{\propbanditzeroreward}[1]{
\begin{proposition}
#1
Assume that in every state there exists an action $a^*(s)$ that achieves the maximum reward $r^* = \max_{s,a}r(s,a)$, i.e., $r^* = \bar r(s,a^*(s)), \; \forall s$ and suppose that \ucrl uses the true expected rewards in its internal computations. For any initial state $(s)$ and any $T \geq 1$ the regret of \ucrl on such MDP is exactly zero jointly for all timesteps.
\end{proposition}
}

\newcommand{\propcontextualbandit}[1]{
\begin{proposition}
#1
Assume that \ucrl is run on an MDP where $P(s' \mid s,a) = \mu(s'), \; \forall s,a,s'$, i.e., the successors are sampled from a fixed underlying distribution. 
For any initial state $s \in S$ and any $T \geq 1$, with probability $1-\delta-o(\delta)$ the regret of \ucrl is bounded by 
\begin{equation*}
\tilde O\( S\sqrt{AT} + DS^3A^2\sqrt{T_{hit}} \).
\end{equation*}
\end{proposition}
}

\newcommand{\ubev}[0]{\textsc{Ubev} }
\newcommand{\ubevs}[0]{\textsc{Ubev-S} }
\newcommand{\ucrl}[0]{\textsc{Ucrl2} }
\newcommand{\sqrtucrlPsbar}[1]{\ensuremath\sqrt{\frac{14S\log(\frac{2At_k}{\delta})}{\max\{1,N_k(#1,\pi_k (#1))\}}}}
\newcommand{\sqrtucrlRsbar}[1]{\ensuremath\sqrt{\frac{7\log(\frac{2At_k}{\delta})}{2\max\{1,N_k(#1,\pi_k(#1))\}}}}

\def\sumall{\ensuremath{\sum_{ \substack{k \in G \\t \in [H] \\ (s,a)}}}}
\def\MC{\ensuremath{\mathcal M_{C } }}

\newcommand{\magicphik}[1]{\phi \(\overline s_{#1},\pi_k(\overline s_{#1},t)\)}

\newcommand{\phiplus}[0]{4\sqrt{S}H^2 \max_{(s',t')} \phi(s',\pi_k(s',t'))}

\newcommand{\pseudoregret}[1]{\ensuremath{\Delta}}

\icmltitlerunning{Identifying Bandit Structure in MDPs}

\allowdisplaybreaks[4]

\begin{document}

\twocolumn[
\icmltitle{Problem Dependent Reinforcement Learning Bounds Which Can Identify Bandit Structure in MDPs}



\icmlsetsymbol{equal}{*}

\begin{icmlauthorlist}
\icmlauthor{Andrea Zanette}{stanford}
\icmlauthor{Emma Brunskill}{stanford}
\end{icmlauthorlist}

\icmlaffiliation{stanford}{Stanford University, Stanford, California}
\icmlcorrespondingauthor{Andrea Zanette}{zanette@stanford.edu}

\icmlkeywords{Reinforcement Learning, Bandits, PAC, Regret}

\vskip 0.3in
]



\printAffiliationsAndNotice{}  

\begin{abstract}
In order to make good decision under uncertainty an agent must learn from observations. To do so, two of the most common frameworks are Contextual Bandits and Markov Decision Processes (MDPs). 
In this paper, we study whether there exist algorithms for the more general framework (MDP) which automatically provide the best performance bounds for the specific problem at hand without user intervention and without modifying the algorithm. In particular, it is found that a very minor variant of a recently proposed reinforcement learning algorithm for MDPs already matches the best possible regret bound $\tilde O (\sqrt{SAT})$ in the dominant term if deployed on a tabular Contextual Bandit problem despite the agent being agnostic to such setting.
\end{abstract}
\section{Introduction}
For reinforcement learning (RL) to realize its huge potential benefit, we must create reinforcement learning algorithms that do not require extensive expertise and problem-dependent fine-tuning to achieve high performance in a particular domain of interest. Much exciting research is advancing this vision, such as alleviating the need for feature engineering using deep neural networks, and making it easier to specify the desired behavior through inverse reinforcement learning and reward 
design \cite{Mnih2013,Apprenticeship}. Here instead we consider the theoretical aspects of a key but understudied issue: what decision process framework to use, and how that choice impacts the resulting performance. 

In reinforcement learning (learning to make good decisions under uncertainty), there are three common frameworks that allow learning from observations: multi-armed bandits (MABs) and contextual MABs, Markov decision processes (MDPs) and partially observable MDPs (POMDPs). Bandits assume that the actions taken do not impact the next state, MDPs assume actions impact the next state but the state is a sufficient statistic of prior history, and POMDPs assume that the true Markov state is latent, and in general the next state can depend on the full history of prior actions and observations. It is known that these three decision process frameworks differ significantly in computational complexity and statistical efficiency. 
In particular, when the decision process model is unknown and an agent must perform reinforcement learning, existing theoretical bounds illustrate that the best results possible in bandits, contextual bandits, MDPs and POMDPs may significantly differ. 
For example there exist \textit{upper bounds} on the regret of algorithms for 
discrete state and action contextual bandits which scale as $\tilde O(\sqrt{SAT})$ (see \cite{BC12}) and \textit{lower bounds} on the regret of algorithms for episodic discrete state and action MDPs which scale as $\Omega(\sqrt{HSAT})$ \cite{OV16}, here indicating there is a gap of at least a factor of $\sqrt{H}$ between the regret possible in the two settings. 
Such work suggests that to obtain good performance, it is of significant interest to have algorithms that either implicitly or explicitly
use the simplest setting (of bandits, MDPs, POMDPs) that captures the domain of interest 
during reinforcement learning, 

As (outside of simulated domains) the true decision process properties are unknown, choosing whether to model a problem using the bandit, MDP or POMDP frameworks is typically far from trivial. A software engineer working on a product recommendation engine may not know whether the product recommendations have a significant impact on the customers' later states and preferences, such that the engineer should model the problem as a MDP instead of a bandit in order to be able to use a reinforcement learning algorithm to learn a policy that best maximizes revenue. This may result in requiring prohibitive amounts of interaction data to learn a good decision policy. Ideally an engineer should be able to write down a problem in a very general way and be confident that the algorithm will inherit the best performance of the underlying domain and problem. 

Here we work to create RL algorithms with strong setting / framework dependent bounds. Our hope is to create reinforcement learning methods that perform as well as the underlying process allows but without the algorithm user having to specify in advance the process framework (bandit / MDP / POMDP) which is often unknown. In doing so we hope to alleviate the burden on the users, allowing them to inherit the benefits of more complex policies if the situation allows, without performance being harmed if the true process is simpler than the one specified.

Precisely here we consider the challenge of creating MDP algorithms that can inherit the best properties of tabular contextual bandits if the RL algorithm is operating in such setting. 
Our aim is similar in motivation to problem dependent theoretical analyses, that seek to provide tighter performance bounds by including an explicit dependence on some property of the domain, such as the mixing rate \cite{AO06}, or the difference in rewards or optimal state-action values \cite{Auer02,AG12,EMM06}.  However, existing problem dependent research has not yet enabled strong process-dependent learning bounds (e.g. bounds that depend on whether the domain is a MDP or a bandit). 
Prior problem dependent results are limited for our setting of interest because they typically make restrictive assumptions on the subset of Markov decision processes for which they hold (e.g., highly mixing for \cite{AO06}), require the user to explicitly provide domain properties \cite{Bartlett09} or the provided bound does not yield strong guarantees when the MDP algorithm is deployed on a simpler bandit process \cite{Maillard14}. A work with more similar intentions to ours is \cite{Bubeck2012} where the authors propose an algorithm whose regret is optimal both for adversarial rewards and for stochastic rewards; by contrast here we consider a change in the learning framework (MDPs vs Bandits).

 Perhaps the most closely related work is the recently introduced contextual decision process research \cite{JKAL17}. 
The authors provide probably approximately correct (PAC) results for generic CDPs as a function of their 
Bellman rank; however their resulting bounds for tabular MDPs and CMABs do not provide the best or near-best PAC bounds (both have a worse dependence on the horizon). 
In contrast our work considers an algorithm for which we can achieve a near-optimal performance on MDPs and the best regret upper bound in the dominant terms for tabular contextual bandits. 

In other words, we can use an MDP RL algorithm and if the real world is a bandit, the MDP RL algorithm automatically scales in performance about as well as a near-optimal algorithm that was designed \textit{specifically} for bandit problems. Precisely, a small variant of \ubev \cite{Dann17} yields a $\sqrt{SAT}$ regret term if the MDP it is acting in is actually a tabular contextual bandit regardless of the prescribed MDP horizon $H$. Prior work in provably efficient RL algorithms \cite{Jaksch10,Dann15,Azar17} provide regret or PAC guarantees which depend on the MDP horizon $H$ or diameter $D$ for episodic and infinite-horizon MDPs, respectively. $H$ is the MDP horizon and is specified to the algorithm. Therefore these analyses do not imply that ``$H$'' can be removed if the $H$-horizon MDP is actually generated from a CMAB problem.

The key insight of our analysis is to show that due to the bandit structure, the optimistic value function converges to the optimal value function fast enough that the regret bound terms due to the MDP framework contribute only to lower order terms with a logarithmic time dependence.  In the rest of the paper, we first outline the setting, introduce the algorithm, and then provide our theoretical results and proofs before discussing future directions. 


\begin{algorithm*}[t!]
   \caption{\ubevs for Stationary Episodic MDPs}
   \label{main:ubevs}
\begin{algorithmic}[1]
   \STATE \textbf{Input}: failure tolerance $\delta \in (0, 1]$
   \STATE $n(s,a) = l(s,a) = m(s',s,a) = 0 \;\; \forall s',s,a\in \mathcal S \times \mathcal S \times \mathcal A; \quad \tilde V_{H+1}(s) = 0 \;\; \forall s \in \mathcal S; \quad   \phi^+ = 0$
   \FOR{$k=1,2,\dots$}
      \FOR{$t=H,H-1,\dots,1$}
     	 \FOR{$s \in \mathcal S$}
	 \FOR{$a \in \mathcal A$}
   \STATE $\phi = \sqrt{\frac{2\ln\ln(\max\{e,n(s,a)\}) + \ln(27HSA/\delta)}{n(s,a)}}$ \\
   \STATE $\hat r = \frac{l(s,a)}{n(s,a)}, \hat V_{next} = \frac{m(\cdot,s,a)^\top \tilde V_{t+1}}{n(s,a)}$
   \STATE $Q(a) = \min \{1,\hat r + \phi \} + \min \{ \max_s \tilde V_{t+1}(s),\hat V_{next} + \min \{ (H-t),  (\rng \tilde V_{t+1}+\phi^+) 
  \}\phi  \}$
  \ENDFOR
   \STATE $\pi_k(s,t) = \argmax_a Q(a); \quad \tilde V_t(s) = Q(\pi_k(s,t)); \quad \phi^+ = \max\{4\sqrt{S}H^2\phi(s,\pi_k(s,t)),\phi^+ \} $
   \ENDFOR
   \ENDFOR
   \STATE
   $s_1 \sim p_0$
   \FOR{t=1,\dots H}
   \STATE $a_t = \pi_k(s_t,t);\quad r_t \sim p_R(s_t,a_t);\quad s_{t+1}\sim p_P(s_t,a_t)$
   \STATE $n(s_t,a_t)++;\quad m(s_{t+1},s_t,a_t)++; \quad l(s_t,a_t) += r_t$
   \ENDFOR
   \ENDFOR
\end{algorithmic}
\end{algorithm*}

\section{Notation and Setup}
A finite horizon MDP is defined by a tuple $\mathcal M = \langle\ \mathcal S,\mathcal A,p,r, H \rangle\ $, where $\mathcal S$ is the state space, $\mathcal A$ is the action space, $p: \mathcal S \times \mathcal A \times \mathcal S \rightarrow \R$ is the transition function where  $p (s'\mid s, a)$ is the probability of transitioning to state $s'$ after taking action $a$ in state $s$. The mean reward function $r : \mathcal{S} \times \mathcal{A} \rightarrow  \R \in [0,1] $ is the average instantaneous reward collected upon playing action $a$ in state $s$, denoted by $r(s,a)$. 
The agent interacts with the environment in a sequence of episodes $k \in [1,\ldots,K]$, each of a horizon of $H$ time steps before resetting. 
As the optimal policy in finite-horizon domains is generally time-step-dependent, on each episode the agent selects a $\pi_k$ which maps states $s$ and timesteps $t$ to actions. 
A policy $\pi_k$ induces a value function for every state $s$ and timestep $t\in[H]$ defined as $V_t^{\pi_k}(s_t) = \E \sum_{i = t}^H r(s_i,\pi_k(s_i,i))$
which is the expected return until the end of the episode (the expectation is over the states $s_i$ encountered in the MDP). We denote the optimal policy with $\pi^*$ and its value function as $V^*_t(s)$ and define the range of a vector $V$: $\rng V \stackrel{def}{=} \max_s V(s) - \min_s V(s)$.

There are multiple formal measures of RL algorithm performance. We focus on regret, which is frequently used in RL and very widely used in bandit research. 
Let the 
regret of the algorithm up to episode $K$ from any sequence of starting states $s_{1k},s_{2k},\dots$ be:
\begin{align}
\text{Regret}(K) \stackrel{def}{=} \sum_k V_1^*(s_{1k}) -  V_1^{\pi_k}(s_{1k}).
\end{align}
Since the policies depend on the history of observations, the regret is a random variable. Here we focus on a high probability bound on the regret. 

We use the $\tilde O(\cdot )$ notation to indicate a quantity that depends on $(\cdot)$ up to a $\polylog$ expression of a quantity at most polynomial in $S,A,T,K,H,\frac{1}{\delta}$.  
We use the $\lesssim, \gtrsim, \simeq$ notation to mean $\leq, \geq, =$, respectively, up to a numerical constant.

\section{Mapping Contextual Bandits to MDPs}
Tabular contextual multi-armed bandits are a generalization of the multiarmed bandit problem. They prescribe a set of contexts or states and the expected reward of an action depends on the state and action, $r(s,a)$. They can be alternatively viewed as a simplification of MDPs in which the next state is independent of the prior state and action. Let \MC{} be an episodic MDP with horizon $H$ which is actually a contextual bandit problem: the transition probability is identical $p(s'|s,a) = \mu(s')$ for all states and actions, where $\mu$ is a fixed stationary distribution over states. Note that when doing RL in a \MC{} the agent does not know the transition model and therefore does not know the MDP can be viewed as a contextual bandit.

\section{UBEV for Stationary MDPs}
In this section we introduce the \ubevs algorithm 
which is a slight variant of \ubev \cite{Dann17}, a recent PAC algorithm designed for episodic non-stationary MDPs. 
Here we focus on a regret analysis due to its popularity in the bandit literature. 


A large fraction of the literature for episodic MDPs considers stationary environments. If the MDP is truly stationary (i.e., with time-independent rewards and transition dynamics) then this assumption can be leveraged to produce $\sqrt{H}$-tighter regret bounds. For the purpose of our analysis on CMABs the rationale for removing the non-stationarity from \ubev is the following: if the MDP is transient the agent cannot ``assume'' that the same state $s$ gives identical expected rewards $r(s,a)$ if visited at different timesteps, say $t_1$ and $t_2$. As a consequence, it would treat the same ``context'' $s$ visited at $t_1$ and $t_2$ as different entities. 
We therefore adapt \ubev to handle stationary MDPs and modify the exploration bonus slightly. 
This second change preserves the original bounds in the MDP setting and enables us to obtain stronger bounds in the bandit setting.
We call the resulting algorithm \ubevs (Algorithm \ref{main:ubevs}). Lines $4$ through $13$ refers to the planning step and lines $14$ through $18$ to the execution of the chosen policy in the MDP. 
\ubevs is a minor variant of \ubev and it can be analyzed in the same way as the original \ubev to obtain a regret bound whose leading order term is $\tilde O (H\sqrt{SAT})$ on a generic (albeit stationary) MDP\footnote{Notice the difference in notation. Here $T$ is the time elapsed; in \cite{Dann17} it is the number of episodes elapsed. The two differ by a factor of $H$.}. We outline such analysis in the appendix (in section \ref{ubevs:StationaryResult}). 
The main difference from \ubev in \cite{Dann17} and \textsc{ubev-S} here is the stated stationarity of the MDP. 
In stationary MDPs the transition dynamics $p(s'|s,a)$ and rewards $r(s,a)$ are assumed to be time-independent for a fixed $(s,a)$ pair.  This allows data aggregation for the same state-action pair $(s,a)$ from different  timesteps $t$ in order to estimate the rewards and system dynamics, as  
seen in lines $2,7,8,17$. 
As a result, \ubevs is more efficient on stationary environments because it does not need to estimate $r$ and $p$ for different timesteps but it will not handle transient MDPs as \ubev. This ultimately leads to a saving of $\sqrt{H}$ in the leading order regret term if the MDPs is time-invariant. 

The other minor change is to make the exploration bonus (Algorithm \ref{main:ubevs} Line 9)  depend on the range of the optimistic value function $(\rng {\tilde V_{t+1}^{}}) \phi(s,a)$ (defined in Algorithm \ref{main:ubevs}) of the successor states. In contrast \ubev used a fixed overestimate $(H-t)\phi(s,a)$. A bonus dependent on the actual $\tilde V_{t+1}^{\pi_k}$ is the typical approach used in similar works (e.g. \cite{Jaksch10,Dann15,Azar17}). The rationale here is that if $\rng \tilde V_{t+1}$ is very small then the agent is not ``too uncertain'' about that transition, hence the exploration bonus should be smaller. Although this does not improve the MDP  regret bound (which only considers a worst-case scenario), better practical performance should be expected and it will have important benefits for our bandit analysis. For the exploration bonus to be valid we require that optimism be guaranteed on any MDP. We ensure this by adding a correction term $\phi^+$ which varies in different $(s,a)$ pairs and is an estimate of the uncertainty of $\rng{\tilde V_{t+1}^{}}$. The correction term $\phi^+$ is continuously updated in line $11$ of Algorithm \ref{main:ubevs} so that $\phi^+$   keeps track of the largest bonus / confidence interval which is related to the least visited $(s,a)$ pair (in subsequent states) under the agent's policy. In the appendix (section \ref{ubevs:OptimismAppendix}) we carefully justify why this choice guarantees optimism on any MDP. This change does not affect the  regret bound for stationary MDPs since our exploration bonus is still upper  bounded by $H \phi(s,a)$ (this is the upper bound used to obtain the result on MDPs).

\section{Theoretical Result}
In this section we present the main result of the paper, which is an upper bound on the regret of \ubevs on \MC{}.
\begin{theorem}
\MainResult{\label{main:MainResult}}{\label{main:CMABresult}}{\label{main:AnyMDPresult}}
\end{theorem}
Notice that equation \ref{main:CMABresult} is obtained by the analysis that we discuss in this main paper 
while equation \ref{main:AnyMDPresult} is the regret bound that \ubevs would achieve in \emph{any} episodic stationary MDP (detailed the appendix). Since \MC{} is an MDP, the tighter bound applies. 

The significance of this result is that the leading order term matches the lower bound $\Omega(\sqrt{SAT})$ previously established for tabular contextual bandit problems. The lower order terms of Equation \ref{main:CMABresult} depend upon $\mu_{min} \stackrel{def}{=} \min_s \mu(s)$, which is the lowest probability of visiting any given context. 

Put differently, for $T$ sufficiently large and not too small $\mu_{min}$, the leading order term dominates and the bound matches the lower bound for contextual bandits up to $\polylog(\cdot)$ factor. Problems where a large $T$ is most critical for the regret are those where the optimal actions are barely distinguishable from the suboptimal ones. Our result shows that in this case there is little penalty for using a more general approach like \ubevs which is designed for MDPs and is unaware of the problem structure. By the time the agent has identified which actions have maximum instantaneous reward the structure of the underlying problem is already clear to the agent.
The key insight to obtain the result of theorem \ref{main:MainResult} is to examine the rate at which the optimistic value function $\tilde V_t^{\pi_k}$ converges to the true one $V_t^*$. While such convergence does not necessarily occur in a generic MDP, the highly mixing nature of contextual bandits ensures that enough information is collected in every context / state that convergence of the value function does occur for all states. The rate of convergence is high enough that the ``price'' for using an MDP algorithm on CMABs gets transferred to lower order terms without any $T$ dependence.

\section{Analysis on \MC{}}
We begin our analysis by looking at the main source of regret for \ubevs when deployed on a generic MDP. We do this to identify the leading order term contributing to the regret. Next, we provide a tighter analysis of such term when the process is a CMAB.

Optimistic RL agents work by computing with high probability an optimistic value function $\tilde V_1^{\pi_k}(s_0)$ for any starting state $s_0$. This overestimates the true optimal value function $V_1^{*}(s_0)$ and allows to estimate the regret of an agent by evaluating the same policy on two different MDPs which get closer and closer to each other as more data is collected:
\begin{align*}
& \text{Regret}(K) \stackrel{def}{=} \sum_k V_1^*(s_0) -  V_1^{\pi_k}(s_0) \\
& \stackrel{Opt.}{\leq} \sum_k \tilde V_1^{\pi_k}(s_0) -  V_1^{\pi_k}(s_0) \\
& \stackrel{}{=} \underbrace{ \sum_{k\leq K}\sum_{t \in [H]}\sum_{s,a} w_{tk}(s,a) \( \tilde r_k(s,a) - r(s,a) \)}_{\tilde O \(\sqrt{SAT} \)} \\
& + \underbrace{\sum_{k\leq K}\sum_{t \in [H]}\sum_{s,a} w_{tk}(s,a) \( \tilde p_k(s,a) - p(s,a) \)^\top \tilde V_{t+1}^{\pi_k}}_{\tilde O \(H\sqrt{SAT} \)}
\numberthis{\label{main:RegretDecomposition}}
\end{align*}
In the above expression the last equality follows from a standard decomposition, see for example lemma E.15 in \cite{Dann17}. We indicated with $\tilde p_k(s,a)$ the optimistic transition probability vector implicitly computed by \ubevs along with the optimistic value function $\tilde V_t^{\pi_k}$. Here $w_{tk}(s,a)$ is the  probability of visiting state $s$ and taking action $a$ there at timestep $t$ of the $k$-th episodes. Finally,  ${\tilde r_k(s,a)}$ is the instantaneous optimistic reward collected upon taking action $a$ in state $s$. 

Below each term we have reported the regret that $\ubevs$ would obtain on a generic MDP. 
Estimating the rewards alone implies a regret contribution of order $\tilde O (\sqrt{SAT})$, which is what a (near) optimal CMAB algorithm achieves. Thus, to obtain a tighter bound on \MC{} we need to address the regret due to the transition dynamics which is of order $\tilde O (H\sqrt{SAT})$ for \ubevs{} on a generic MDP. A careful examination of the proof for that regret bound of that term reveals that $H$ appears because it is a deterministic upper bound on the range of $ \tilde V_t^{\pi_k}$ and $V_t^{*}$. The optimistic value function is a random variable, but under the assumption that $r(s,a) \in [0,1]$ the agent maintains an optimistic estimate of such reward with the same constraint $\tilde r(s,a) \in [0,1]$, leading to $\rng \tilde V_t^{\pi_k} \leq H$ when the rewards are summed over $H$ timesteps; likewise $V_t^{*} \leq H$. As we show next, \MC{} is characterized by  $\rng V_t^* \leq 1$, which means there is not a big advantage for being in one context (i.e., state) versus another. This happens because the agent's current mistake only affects the instantaneous reward; the agent can never make ``costly mistakes'' that lead it to a sequence of contexts / states with low payoff as a result of that mistake as may happen on a generic MDP. Unfortunately this consideration need not be true in the ``optimistic'' MDP that the agent computes, that is, it is not true that $\rng \tilde{V}^{\pi_k} \leq 1$.  However, we can relate $\rng \tilde V_t^{\pi_k}$ to $\rng V_t^*$ and show that $\rng \tilde V_t^{\pi_k}$ is of order $1$ plus a quantity that shrinks fast enough so that the regret contribution due to uncertain system dynamics is of the same order as the rewards plus a term that does not depend on $\sqrt{T}$.

\textbf{Remark}: the convergence of the optimistic value function to the true one is not a property generally enjoyed by these algorithms, see for example \cite{Bartlett09} for an extensive discussion for \ucrl-style approaches in the infinite horizon case. However, said convergence does occur here due to the highly mixing nature of the contextual bandit problem.

\subsection{Range of the True Value Function}
On \MC{} a policy that greedily maximizes the instantaneous reward is optimal. Let $\overline s_t \stackrel{def}{=} \argmax V^*_t(s)$ and $\underline s_t \stackrel{def}{=} \argmin V^*_t(s)$ and recall that the transition dynamics $P(s,a) = \mu$ depends nor on the action $a$ nor on the current state $s$:  
  \begin{equation}
  \label{main:VStarBackup}
    \begin{cases}
      V^*_t(\overline s_t) = \max _a \( r( \overline s_t,a) + \mu^\top V^*_{t+1} \) \\
      V^*_t(\underline s_t) = \max _a \( r( \underline s_t,a) + \mu^\top V^*_{t+1} \)
    \end{cases}
  \end{equation}
Since the rewards are bounded $r(\cdot,\cdot) \in [0,1]$ subtracting the two equations in \ref{main:VStarBackup} yields:
\begin{equation}
\label{main:RngVStar} 
\rng V^*_t \stackrel{}{=}  
\max_a r( \overline s_t,a) -  \max _a r( \underline s_t,a) \leq 1.
\end{equation}

\subsection{Range of the Optimistic Value Function}
Now we relate $\rng \tilde V_t^{\pi_k} $ to $\rng V_t^{*} $ by a quantity that is naturally shrinking.
Our reasoning assumes that we are outside the failure event so that confidence intervals hold (confidence intervals are essentially the same as \ubev and are discussed in the appendix in section \ref{ubevs:FailureEventAppendix}). We use the notation $n_k(s,a)$ to indicate the number of visit to the $(s,a)$ pair at the beginning of the $k$-th episode.
\begin{lemma}
\label{main:RngVBound}
If \ubevs is run on \MC{} then outside of the failure event it holds that:
\begin{equation}
\label{main:EqRngVBound}
\rng \tilde V^{\pi_k}_t \leq 1  + \tilde O \( \frac{H\sqrt{S}}{\sqrt{\min_{(s',t')} n_k(s',\pi_k(s',t'))}} \).
\end{equation}
\end{lemma}
\begin{proof}
We denote by $\hat p_k(s,a)$ the maximum likelihood vector for the transitions from $(s,a)$. 
For simplicity redefine $\underline s_{tk} = \argmin_s \tilde V_t^{\pi_k}(s) $ and $\overline s_{tk} = \argmax_s \tilde V_t^{\pi_k}(s) $.
Neglecting the reward $\tilde r_k(s,\pi_k(s,t))$ and the optimistic bonus $\phi$ while planning at timestep $t$ (line $9$ of the algorithm) yields a lower bound on the optimistic value function: 
\begin{align*}
\min_s \tilde V_t^{\pi_k}(s) & \stackrel{def}{=} \tilde V_t^{\pi_k}(\underline s_{tk}) \geq  \hat p_{k}(\underline s_{tk},\pi_k(\underline s_{tk},t))^\top \tilde V_{t+1}^{\pi_k}.
\numberthis \label{main:minVopt}
\end{align*}
Recalling that $\tilde r(s,a) \leq 1$, an upper bound on $\tilde V_t^{\pi_k}$ can also be obtained (from planning in line $9$):
\begin{align*}
& \max_s \tilde V_t^{\pi_k}(s)  \stackrel{def}{=} \tilde V_t^{\pi_k}(\overline s_{tk}) \\
& \leq \underbrace{1}_{\substack{\text{Reward}}} + \hat p_k(\overline s_{tk},\pi_k(\overline s_{tk},t))^\top \tilde V_{t+1}^{\pi_k} + \underbrace{H\magicphik{k}}_{\substack{\text{Bonus} \\ \text{(Overestimate)}}}.
\numberthis \label{main:maxVopt}
\end{align*}
Subtracting \ref{main:minVopt} from \ref{main:maxVopt} yields $(a)$ below:
\begin{align*}
& \rng \tilde V^{\pi_k}_t \stackrel{def}{=}  \max_s \tilde V_t^{\pi_k}(s) - \min_s \tilde V_t^{\pi_k}(s) \leq \\
\stackrel{(a)}{\leq}  & 1 +  \( \hat p_k(\overline s_{tk},\pi_k(\overline s_{tk},t))^\top - \hat p_{k}(\underline s_{tk},\pi_k(\underline s_{tk},t))^\top \) \tilde V_{t+1}^{\pi_k}  \\
& + H\magicphik{k} \\
 \stackrel{(b)}{\leq}  &1+\| \hat p_k(\overline s_{tk},\pi_k(\overline s_{tk},t))-\hat p_{k}(\underline s_{tk},\pi_k(\underline s_{tk},t)) \|_1 \|\tilde V_{t+1}^{\pi_k}\|_{{}_\infty}  \\
 & + H\magicphik{k} \\
\stackrel{(c)}{\leq}  & 1 + H \| \hat p_k(\overline s_{tk},\pi_k(\overline s_{tk},t)) - \mu \|_1 \\
& + H \| \hat p_k(\underline s_{tk},\pi_k(\underline s_{tk},t)) - \mu \|_1   + H\magicphik{k} \\
\numberthis{\label{main:FinalEqnInProp}}
\end{align*}
In $(b)$ we used Holder's inequality and the hard bound $\rng \tilde V_{t+1}^{\pi_k} \leq H$ coupled with the triangle inequality for step $(c)$. Before continuing the development we pause and notice that we have upper bounded $\rng \tilde V_t^{\pi_k}$ by $1$ plus two concentration terms (for the transition probabilities) and the optimistic bonus, which are quantities that are shrinking on \MC{}.
In particular, being outside of the failure event ensures a bound on the system dynamics (this is made precise by referring to the concentration inequality of the failure event $F^{L_1}_k$ as explained in our appendix in section \ref{ubevs:FailureEventAppendix}):
\begin{equation}
\label{main:SoverN}
\| \hat p_k(s,a)) - \mu \|_1 = \tilde O \(\sqrt{\frac{S}{n_k(s,a)}} \) 
\end{equation}
The exploration bonus defined in line $7$ of algorithm 
\ref{main:ubevs} is also similar in magnitude:
\begin{equation}
\label{main:SHoverN}
H\phi(s,a) = \tilde O \( \frac{H}{\sqrt{n_k(s,a)}} \)
\end{equation}
By definition, $\min_{(s',t')} n_k(s',\pi_k(s',t')) \leq n_k(s,\pi_k(s,t))$ for any $s,t$ pair which allows us to combine equation \ref{main:SoverN} and \ref{main:SHoverN} above to rewrite \ref{main:FinalEqnInProp} as: 
\begin{equation}
1 + \tilde O \( H{\frac{\sqrt{S}+1}{\sqrt{\min_{(s',t')} n_k(s',\pi_k(s',t'))}}}  \) \\
\end{equation}
which can be simplified to obtain the statement.
\end{proof}

\subsection{Regret Analysis on \MC{}}
Lemma \ref{main:RngVBound} shows that the optimistic value function on \MC{} is of order $1$ plus a quantity which is related to the confidence interval of the least visited $(s,a)$ pair under the policy selected by the agent. On \MC{} we know that the states are sampled from $\mu$. This ensures  that all states are going to be visited at a linear rate so that $ \min_{(s',t')} n_k(s',\pi_k(s',t'))$ must be increasing at a linear rate. The above consideration together with lemma \ref{main:RngVBound} allows us to sketch the analysis that leads to the result of theorem \ref{main:MainResult}.
\subsubsection{Regret Decomposition}
Outside of the failure event we can use optimism to justify the first inequality below that leads to the regret decomposition for the first $K$ episodes:
\begin{align*}
& \textsc{Regret}(K) \stackrel{def}{=}  \sum_{k = 1}^K V^{\pi^*}_1(s) - V^{\pi_k}_1(s) \\
& \stackrel{Optimism}{\leq} \sum_{k = 1}^K \tilde V^{\pi_k}_1(s) - V^{\pi_k}_1(s) \\
& = \sum_{k=1}^K \sum_{t \in [H]} \sum_{(s,a)} w_{tk}(s,a) \Biggm( \underbrace{\(\tilde r(s,a) - r(s,a) \)}_{\text{Reward Estimation and Optimism}} + \\
& + \underbrace{\(\tilde p(s,a) - \hat p(s,a) \)^\top \tilde V_{t+1}^{\pi_k}}_{\text{Transition Dynamics Optimism}} +  \underbrace{\(\hat p(s,a) - p(s,a) \)^\top V_{t+1}^{*}}_{\text{Transition Dynamics Estimation}} \\
& +  \underbrace{\(\hat p(s,a) - p(s,a) \)^\top \(\tilde V_{t+1}^{\pi_k} - V^*_{t+1} \)}_{\text{Lower Order Term}}  \Biggm). \\
\numberthis{\label{main:AdvancedRegretDecomposition}}
\end{align*}
The decomposition is standard in recent RL literature \cite{Azar17,Dann17}.
\subsubsection{The ``Good'' Episodes on \MC{}}
In the original paper \cite{Dann17}, the authors introduce the notion of ``nice'' and ``friendly'' episodes to relate the probability of visiting a state-action pair $w_{tk}(s,a)$ to the actual number of visits there $n_k(s,a)$ (the latter is a random variable). Here we do a similar distinction directly for a regret analysis (as opposed to a PAC analysis) and we leverage the structure of \MC{}. 
In particular we partition the set of all episodes into two, namely the set $G$ of \emph{good episodes} and the set of episodes that are ``not good''. 
Under good episodes we require that:
\begin{equation}
\label{main:GoodEpisodes}
n_k(s,a) \geq \frac{1}{4} \sum_{i< k}\sum_{\tau\in [H]}w_{\tau i}(s,a)
\end{equation}
holds true for \emph{all} states $s$ and actions $a$ chosen by the agent's policy. In other words, we require that the number of visits $n_k(s,a)$ to the $(s,a)$ pair is at least $\frac{1}{4}$ times its expectation. In lemma \ref{ubevs:RegretNonGoodEpi} in the appendix we examine the regret under non-good episodes, which can be bounded by $\tilde O (\frac{SAH^2}{\mu_{min}})$.

\subsubsection{Regret Bound for the Optimistic Transition Dynamics (Leading Order Term)}


Equipped with lemma \ref{main:RngVBound} we are ready to bound the leading order term contributing to the regret under good episodes. This is the regret due to the optimistic transition dynamics which appear in equation \ref{main:AdvancedRegretDecomposition}. While planning for state $s$ and timestep $t$ (see line $9$ of Algorithm \ref{main:ubevs}), \ubevs implicitly finds an optimistic transition dynamics $\tilde p_k(s,a)$. In particular the ``optimistic'' MDP satisfies the following upper bound on $\tilde p_k(s,a)^\top \tilde V_{t+1}^{\pi_k}$:
\begin{equation}
\label{main:OptimisticPlanningBound}
 \stackrel{line \; 9}{\leq} \hat p_k(s,a)^\top \tilde V_{t+1}^{\pi_k} + (\rng \tilde V_{t+1}^{\pi_k} + \phi^+) \phi_{tk}(s,a).
\end{equation}
Notice that line $9$ of the algorithm provides additional constraints enforced by taking $\min\{\cdot,\cdot\}$, but equation \ref{main:OptimisticPlanningBound} always remains an upper bound.
Rearranging the inequality above and summing over the ``good episodes'', the timesteps $t\in[H]$ and all the $(s,a)$ pairs yields an upper bound on the regret due to the optimistic transition dynamics that appears in equation \ref{main:AdvancedRegretDecomposition}:
\begin{align*}
& \sum_{k\in G} \sum_{t \in [H]} \sum_{(s,a)} w_{tk}(s,a) \(\tilde p_k(s,a) - \hat p_k(s,a) \)^\top \tilde V_{t+1}^{\pi_k}  \\
& \leq \sum_{k\in G} \sum_{t \in [H]} \sum_{(s,a)} w_{tk}(s,a)(\rng \tilde V_{t+1}^{\pi_k} + \phi^+) \phi_{tk}(s,a).
\numberthis{\label{main:IntermediateOpt}}
\end{align*}
Next, notice that the correction factor $\phi^+$ is updated in line $11$ of the algorithm and depends on the state with the lowest visit count $\min_{(s',t')} n_k(s',\pi_k(s',t'))$. This implies the following upper bound on $\phi^+$.
\begin{equation}
\phi^+ \lesssim \frac{H^2\sqrt{S}}{\sqrt{\min_{(s',t')} n_k(s',\pi_k(s',t'))}}\polylog(\cdot).
\end{equation}
At this point we can substitute the definition of $\phi_{tk}(s,a)$ (line 7 of Algorithm \ref{main:ubevs}) and put all the constants and logarithmic quantities in $\polylog(\cdot)$ to upper bound \ref{main:IntermediateOpt} as follows:
\begin{align*}
& \lesssim \sum_{k\in G} \sum_{t \in [H]} \sum_{(s,a)} w_{tk}(s,a) \frac{\rng \tilde V_{t+1}^{\pi_k}}{\sqrt{n_k(s,a)}} \polylog(\cdot) \\
& + \sum_{k\in G} \sum_{t \in [H]} \sum_{(s,a)}\frac{ w_{tk}(s,a) \sqrt{S}H^2 \polylog(\cdot)}{\sqrt{\min_{(s',t')} n_k(s',\pi_k(s',t')) \times n_k(s,a)}}. \\
\numberthis{}
\end{align*}
Finally we substitute lemma \ref{main:RngVBound}:
\begin{align*}
& \lesssim \underbrace{\sum_{k\in G} \sum_{t \in [H]} \sum_{(s,a)} w_{tk}(s,a) \frac{1}{\sqrt{n_k(s,a)}}}_{\text{Leading Order Term}} \polylog(\cdot) \\
& + \underbrace{\sumall \frac{w_{tk}(s,a) \sqrt{S}H^2\polylog(\cdot)}{\sqrt{\min_{(s',t')} n_k(s',\pi_k(s',t')) \times n_k(s,a)}}}_{\text{Lower Order Term}}. \\
\numberthis{}
\end{align*}
and apply Cauchy-Schwartz to get (omitting $\polylog(\cdot)$ factors):
\begin{align*}
& \sqrt{\underbrace{\sum_{k\in G} \sum_{t \in [H]} \sum_{(s,a)} w_{tk}(s,a)}_{\leq T}}\sqrt{\underbrace{\sum_{k\in G} \sum_{t \in [H]} \sum_{(s,a)} \frac{w_{tk}(s,a)}{n_k(s,a)}}_{\tilde O(SA)}} + \\
& \sqrt{S}H^2 \sqrt{\underbrace{\sumall \frac{w_{tk}(s,a)}{n_k(s,a)}}_{\tilde O(SA)}} \sqrt{\underbrace{\sumall \frac{w_{tk}(s,a)}{\displaystyle \min_{(s',t')} n_k(s',\pi_k(s',t')) }}_{(\star)}}.
\numberthis{}
\end{align*}
The sum of the ``visitation ratios'' $ \frac{w_{tk}(s,a)}{n_k(s,a)}$ under good episodes can be bounded in the usual way by $\tilde O(SA)$ by using a pigeonhole argument and will not be discussed further (details are in the appendix). 
To bound ${(\star)}$ we need to work a little more. The main problem is that the ratio 
\begin{equation}
\label{main:KeyVisitRatio}
\frac{w_{tk}(s,a)}{\min_{(s',t')} n_k(s',\pi_k(s',t'))}
\end{equation}
is a ratio between the visitation probability of a certain state $(s,a)$ pair and the visit count of a \emph{different} pair. For a general MDP these two quantities are not related as there can be states that are clearly suboptimal and are visited finitely often by PAC algorithms. As a result, $\sqrt{(\star)}$ can grow like $\sqrt{T}$ and it is not a lower order term. This is the key step where we leverage the underlying structure of the problem. With contextual bandits all contexts are going to be visited with probability at least $\mu_{min}$. Since the analysis is under good episodes, for a fixed $(s',t')$ pair we know that $n_k(s',\pi_k(s',t'))$ must increase by at least $\frac{1}{4}\mu_{min}$ every episodes. There are only $S\times A$ possible candidates for the $(s',a')$ pair with the lowest visit count. Recalling $\sum_{t \in [H]} \sum_{(s,a)} {w_{tk}(s,a)} = H$, the final result then follows by pigeonhole (the computation is in the appendix).
\begin{align}
(\star) = {\sum_{k\in G} \frac{H}{\min_{(s',t')} n_k(s',\pi_k(s',t'))}} = \tilde O\({\frac{SAH}{\mu_{min}}}\).
\end{align}
This completes the sketch of the regret bound for the ``Optimistic Transition Dynamics'' with a regret contribution of order:
\begin{equation}
\tilde O \(\sqrt{SAT} + \sqrt{SH}H^2 \times \frac{SA}{\sqrt{\mu_{min}}} \).
\end{equation}
%
%
\textbf{Remark}: Although for simplicity we conduct here the analysis for the regret only, \ubevs is still a uniformly-PAC algorithm and strong PAC guarantees can be obtained on \MC{} as well. The analysis for the regret due to  the rewards, the estimation of the transition dynamics and the lower oder term can be found in the appendix. Together with the regret in non-good episodes they imply the regret bound of theorem \ref{main:MainResult}.

\section{Discussion, Related Work and Future Work}
A natural question is whether there is something special about the \ubev algorithm, or if other MDP RL algorithms with theoretical bounds can also be shown to have provably better or optimal regret bounds on contextual bandit problems. While we focused on \ubev because it matched (in the dominant terms) the best regret bounds for contextual bandits when run in such settings, we do think other MDP algorithms can yield strong (though not optimal) regret bounds when run in contextual bandits.
For example, \cite{JKAL17} proposes \textsc{Olive}, a probably approximately correct algorithm with bounds for a broad number of settings which can potentially adapt to a CMAB problem if the Bellman rank is known. If the bellman rank is not known in advance (as is our case) a way around this issue is to use the ``doubling trick''. However, the resulting PAC bound of \textsc{Olive} on CMABs would scale in a way which is suboptimal in $H$.
Another interesting candidate for our analysis on CMABs is given in \cite{Bartlett09} the authors propose \textsc{Regal}, a \textsc{Ucrl2}-variant which can potentially achieve a $\tilde O(S\sqrt{AT})$ bound on CMABs while retaining a worst-case $\tilde O(DS\sqrt{AT})$ regret in generic MDPs (here $D$ is the MDP diameter). The simplification on CMABs follows directly from the computation of the span (which is equivalent to the range here) of the optimal bias vector. Still, this result is not completely satisfactory because the lower bound is not achieved and \textsc{Regal} \emph{must know} the range of the bias vector in advance. Another noteworthy variant of \ucrl is discussed in \cite{Maillard14}. There the authors introduce a new norm and its dual (instead of the classical $1$-norm and $\infty$-norm, respectively) to better capture the effect of the MDP transition dynamics. The result that they obtain does depend on a measure of the MDP complexity (constant $C$ in their regret bound). This is essentially the variance of the value function, so $C = O(1)$ on CMABs; despite moving in the right direction, the resulting bound is still of order $\tilde O(DS\sqrt{AT})$ on CMABs.
 
By contrast, our analysis of vanilla \ucrl \cite{Jaksch10} (see appendix \ref{ucrl:Appendix} for extensive details) shows an improved regret bound of $\tilde O (S\sqrt{AT})$ if \ucrl is run on CMABs which is better (although not optimal) than the \ucrl worst-case bound for MDPs $\tilde O (DS \sqrt{AT})$. The key insight to obtain this result is that the MDP diameter $D$ is an upper bound to a key quantity in the analysis of \ucrl{}, and can be more tightly bounded in contextual bandit domains. This analysis suggests that if an algorithm for infinite-horizon MDPs is constructed using $\sqrt{S}$-tighter confidence intervals like in \ubev or \textsc{Ucbvi} from \cite{Azar17} then a bound of order $\tilde O (\sqrt{SAT})$ should be achievable on an infinite horizon \MC{}.

This work raises a number of interesting questions, in particular whether similar results are possible for other pairings of algorithms and domains: can we have algorithms designed for partially observable reinforcement learning that inherit the best performance of the setting they operate in, whether it is a bandit, contextual bandit, MDP or POMDP?  As a step towards such exploration, we analyzed whether a MDP RL algorithm operating in a multi-armed bandit could match the upper bound on regret for such settings. In a multi-armed bandit there are no states, and the reward is solely a function of the arm (action) played. Regret for MABs must scale at least as $\Omega(\sqrt{AT})$, the lower bound for such setting. In our preliminary investigations, our analysis of \ucrl when operating in a MAB (still in section \ref{ucrl:Appendix} in the appendix) yielded an additional $\sqrt{S}$ dependence. It is a very interesting question whether existing or new MDP algorithms that explicitly or implicitly perform state aggregation \cite{BayesianClustering,InfinitePOMDPs} can yield a performance that matches the dominant terms of a bandit-specific regret analysis. Another important question is whether similar analyses are possible for reinforcement learning algorithms designed for very large or infinite state spaces, as well as an empirical investigation to see whether existing RL algorithms for more complex settings  experimentally match algorithms designed for simpler settings when executing in said simpler settings. 

Finally, our analysis for \ubevs highlights a dependence on the minimum visitation probability $\mu_{min}$ which is absent in bandit analyses. We think that this can be avoided by a more careful design of the exploration bonus that re-weights the next-state uncertainty by the transition probability estimated empirically, see for example \cite{Dann15,Azar17}. For simplicity in this paper we focused on tabular bandits and therefore \ubevs cannot handle general Contextual Bandits which use function approximations (e.g, \cite{Abbasi11}).

\section{Conclusion}
The ultimate goal of Reinforcement Learning is to design algorithms that can learn online and achieve the best performance afforded by the difficulty of the underlying domain. In this work we have introduced a minor variant of an existing RL algorithm that automatically provides strong regret guarantees whether it is deployed in a MDP or if the domain actually belongs to a simpler setting, a tabular contextual bandit, matching the lower bound in the dominant terms in the second setting. Note that the algorithm is not informed of this structure.  
This work suggests that already existing RL  algorithms can inherit tighter theoretical guarantees if the domain turns out to have additional structure and yields many interesting next steps for the analysis and creation of algorithms for other settings, particularly the function approximation case. 
\section*{Acknowledgements}
Christopher Dann and the anonymous reviewers are acknowledged for providing very useful feedback which improved the quality of this paper.
\bibliography{rl}
\bibliographystyle{icml2018}

\onecolumn
\appendix

\section{\ubevs for Stationary Environments}
We mostly use the same notation as in \cite{Dann17} and provide the supporting results for \ubevs. 
The assumption of stationary environment is enforced through time aggregation. 
Let $n_{tk}(s,a)$ be the visit count to state-action $(s,a)$ at timestep $t$ up to the start of the $k$-th episode and let $w_{tk}(s,a)$ be the probability of visiting state $s$ and taking action $a$ there at timestep $t$ during the $k$-th episode. Then we defined the corresponding aggregated quantities as:
\begin{equation}
\label{saggregation}
n_k(s,a) \stackrel{def}{=} \sum_{t \in [H]} n_{tk}(s,a).
\end{equation} 
and 
\begin{equation}
\label{swaggregation}
w_k(s,a) \stackrel{def}{=} \sum_{t \in [H]} w_{tk}(s,a).
\end{equation} 

\subsection{Failure Events and Their Probabilities}
\label{ubevs:FailureEventAppendix}
The analysis of the ``failure events'' can be carried out in a way identical to \cite{Dann17}.
In particular we use the same ``failure events'' $F_k^{N},F_k^{CN},F_k^{V},F_k^{P},F_k^{L1},F_k^{R}$ defined in section E.2 in the appendix of \cite{Dann17} but with $n_{tk}(s,a)$ replaced by $n_{k}(s,a)$ whenever it appears. We notice that with \ubevs we could potentially save a factor of $H$ in each argument of the $\log$ terms that appears in each concentration inequality because we do not need to do a final union bound over the $H$ timesteps, resulting in slightly tighter concentration inequalities. 
The total failure probability of \ubevs can then be upper bounded by $\delta$ by using Corollary E.1,E.2,E.3,E.4,E.5 in \cite{Dann17} (still with $n_{tk}(s,a)$ replaced by $n_{k}(s,a)$). If during the execution of \ubevs none of  $F_k^{N},F_k^{CN},F_k^{V},F_k^{P},F_k^{L1},F_k^{R}$ occur in any episode $k$ we say that that we are \emph{outside of the failure event}.

\subsection{The ``Good'' Set}
\label{app:Ltk}
We now introduce the set $L_{tk}$. The construction is due to \cite{Dann17} although we modify it here for our to handle the regret framework (as opposed to PAC) under stationary dynamics. The idea is to partition the state-action space at each episode into two episodes, the set of episodes that have been visited sufficiently often (so that we can lower bound these visits by their expectations using standard concentration inequalities) and the set of $(s,a)$ that were not visited often enough to cause high regret. In particular:
\begin{definition}[The Good Set]
\label{def:TheGoodSet}
The set $L_{k}{}$ is defined as:
\begin{equation}
\label{eqn:TheGoodSet}
L_{k} \stackrel{def}{=} \Big\{ (s,a) \in \mathcal S \times \mathcal A : \frac{1}{4}\sum_{j \leq k} w_{j}(s,a) \geq H\ln\frac{9SA}{\delta}\Big\}.
\end{equation}
\end{definition}
The above definition enables the following lemma that relates the number of visits to a state to its expectation:
\begin{lemma}[Visitation Ratio]
\label{lem:VisitationRatio}
Outside the failure event if $(s,a) \in L_{k}$ then
\begin{equation}
\label{eqn:VisitationRatio}
n_{k}(s,a) \geq \frac{1}{4}\sum_{j\leq k} w_{j}(s,a)
\end{equation}
holds.
\end{lemma}
\begin{proof}
Outside the failure event $F^N$ (see \cite{Dann17}) justifies the first passage below:
\begin{align}
n_k(s,a) & \geq  \frac{1}{2}\sum_{j\leq k}w_{j}(s,a) - H\ln\frac{9SA}{\delta} \\
& = \frac{1}{4}\sum_{j\leq k}w_{j}(s,a) +  \frac{1}{4}\sum_{j\leq k}w_{j}(s,a) - H\ln\frac{9SA}{\delta}  \geq \frac{1}{4}\sum_{j\leq k}w_{j}(s,a).
\end{align}
while the second inequality holds because $(s,a) \in L_{k}$ by assumption.
\end{proof}
Finally, the following lemma ensures that if $(s,a) \not \in L_{k}$ then it will contribute very little to the regret:
\begin{lemma}[Minimal Contribution]
\label{lem:MinimalContribution}
It holds that:
\begin{align}
\sum_{k=1}^K \sum_{t=1}^H \sum_{(s,a) \not \in L_{k}} w_{tk}(s,a) = \tilde O \(SAH\)	
\end{align}
\end{lemma}

\begin{proof}
By definition \ref{def:TheGoodSet}, if $(s,a) \not \in L_{k}$ then
\begin{equation}
\label{eqn:minimalcontri}
\frac{1}{4}\sum_{t \in [H]}\sum_{j \leq k} w_{tj}(s,a) < H\ln\frac{9SA}{\delta}
\end{equation} 
holds. Now sum over the $(s,a)$ pairs not in $L_{tk}$, the timesteps $t$ and episodes $k$ to obtain:
\begin{align}
\sum_{k=1}^K \sum_{t=1}^H \sum_{(s,a) \not \in L_{tk}} w_{tk}(s,a) =  \sum_{s,a} \sum_{t=1}^H \sum_{k=1}^K w_{tk}(s,a) \1\{(s,a)\not \in L_{tk}\} \leq \sum_{s,a} \( 4H\ln\frac{9SA}{\delta}\)= \tilde O \( SAH \)
\end{align}
\end{proof}

\subsection{Ensuring Optimism for \ubevs on Stationary Episodic MDPs}
\label{ubevs:OptimismAppendix}
One of the limitation of \ubev as described in \cite{Dann17} is that the exploration bonus $(H-t) \phi$ does not explicitly depend on (the range of) the value function of the successor but only on its upper bound $(H-t)\phi$, leading to an ``excess of optimism'' in certain classes of problems. To remedy this, we propose to use $\rng \tilde V^{\pi_k}_{t+1}$ instead of $H-t$. While performing optimistic planning to compute $\tilde V^{\pi_k}_{t+1}$, however, it is not guaranteed that $\rng \tilde V^{\pi_k}_{t+1} \geq \rng V^{*}_{t+1}$ and optimism may not be guaranteed. To remedy this we add the correction term $\phi^+$ as described in the main text so that our exploration bonus for the system dynamics reads:
\begin{equation}
\label{ubevs:fullbonus}
\min\{H-t,\rng \tilde V^{\pi_k}_{t+1} + \phi^+\} \phi.
\end{equation}
For this to be a valid exploration bonus we need to show it still guarantees optimism.
To this aim we begin with the following lemma which guarantees that $\phi^+$ accounts for the potentially inaccurate estimate of the value function.

\begin{lemma}
\label{ubevs:extrabonus}
Outside of the failure event $\forall s,t,k$ it holds that:
\begin{equation*}
 \tilde V_t^{\pi_k}(s) - V_t^{\pi_k}(s)  \leq  \phiplus \stackrel{def}{=}\phi^+
\end{equation*}
\end{lemma}
\begin{proof}
Outside of the failure event it holds that:
\begin{align*}
\tilde V_t^{\pi_k}(s) - V_t^{\pi_k}(s) & \stackrel{a}{=} \E \sum_{i=t}^H \( \tilde r_i(s_i,a_i) - r_i(s_i,a_i) \) +  \( \tilde p_i(s_i,a_i) - p_i(s_i,a_i) \)^{\top}\tilde V_{t+1}^{\pi_k} \\
& \stackrel{b}{=}  \E \sum_{i=t}^H \( \tilde r_i(s_i,a_i) - \hat r_i(s_i,a_i) \) + \( \hat r_i(s_i,a_i) - r_i(s_i,a_i) \) +  \( \tilde p_i(s_i,a_i) - p_i(s_i,a_i) \)^{\top}\tilde V_{t+1}^{\pi_k} \\
& \stackrel{c}{\leq}  \E \sum_{i=t}^H 2\phi(s_i,a_i) +  \( \tilde p_i(s_i,a_i) - \hat p_i(s_i,a_i) \)^{\top}\tilde V_{t+1}^{\pi_k} +  \( \hat p_i(s_i,a_i) - p_i(s_i,a_i) \)^{\top}\tilde V_{t+1}^{\pi_k} \\
& \stackrel{d}{\leq}  \E \sum_{i=t}^H 2\phi(s_i,a_i) +  H\phi(s_i,a_i) + \| \hat p_i(s_i,a_i) - p_i(s_i,a_i) \|_1 \| \tilde V_{t+1}^{\pi_k} \|_{\infty} \\
& \stackrel{e}{\leq}  \E \sum_{i=t}^H 2\phi(s_i,a_i) + H\phi(s_i,a_i) + 4\sqrt{S}H\phi(s_i,a_i) \\
& \stackrel{f}{\leq}  \phiplus \stackrel{def}{=}\phi^+.
\end{align*}
\begin{enumerate}[label=(\alph*)]
\item using lemma E.15 in \cite{Dann17}
\item by adding and subtracting $\hat r_i$
\item by adding and subtracting $\hat p_i$ and using the fact that we are outside the failure event $F^{R}_k$ for the rewards and that the confidence interval for the rewards is the same as the exploration bonus $\phi(\cdot,\cdot)$
\item by Holder's inequality and using again the upper bound $H\phi(s,a)$ for the exploration bonus for the system dynamics \ref{ubevs:fullbonus}
\item since we are outside the failure event for the transition probabilities $F^{L1}$ and $\| \tilde V^{\pi_k}_{t+1} \|_{\infty} \leq H$
\item by taking max
\end{enumerate}
\end{proof}

Lemma \ref{ubevs:extrabonus} provides a tool to estimate the uncertainty in the value of the policy. We use this to construct an extra bonus to overestimate the range of the value function (this is needed in lemma \ref{ubevs:OptimisticPlanning} to guarantee optimism). 
\begin{lemma}
\label{ubevs:rangeok}
If $\tilde V_t^{\pi_k}(s) \geq V_t^{*}(s)$ for all states then outside of the failure event it holds that:
\begin{equation}
\rng \tilde V_t^{\pi_k} +\phi^+ \geq \rng V_t^{*}
\end{equation}
\end{lemma}
\begin{proof}
\begin{align*}
\rng \tilde V_t^{\pi_k} + \phi^+ & = \max_s \tilde V_t^{\pi_k}(s) - \min_s \tilde V_t^{\pi_k}(s) + \phi^+ \\
& \geq  \max_s \tilde V_t^{\pi_k}(s) - \tilde V_t^{\pi_k}(\argmin_s V_t^{*} (s)) + \phi^+ \\
& \geq  \max_s \tilde V_t^{\pi_k}(s) - V_t^{\pi_k}(\argmin_s  V_t^{*} (s)) \\
& \geq  \max_s \tilde V_t^{\pi_k}(s) - V_t^{*}(\argmin_s  V_t^{*} (s)) \\
& \geq  \max_s V_t^{*}(s) - \min_s V_t^{*}(s) = \rng{V_t^{*}} \\
\end{align*}
where the middle inequality follows from lemma \ref{ubevs:extrabonus}. 
\end{proof}

\begin{lemma}
\label{ubevs:OptimisticPlanning}
Outside of the failure event \ubevs ensures optimism for all timesteps $t$, states $s$ and episodes $k$:
$$\tilde V_t^{\pi_k}(s) \geq V_t^{*}(s), \quad \forall s,t.$$
\end{lemma}
\begin{proof}
We proceed by induction.
By construction of the algorithm, the computed policy satisfies $\forall s,t,k$:
\begin{align}
\label{eviubevsactual}
\tilde V_t^{\pi_k}(s) & = \max_a \( \min(1,\hat r_k(s,a) + \phi(s,a)) + \min(\max{\tilde V^{\pi_k}_{t+1}},\hat p_k(s,a)^T \tilde V_{t+1}^{\pi_k} + \min(\rng{\tilde V_{t+1}^{\pi_k}} + \phi^+,H-t)\phi(s,a)) \).
\end{align}
If the second minimum in \ref{eviubevsactual} is attained by $\max \tilde V_{t+1}^{\pi_k}$ then optimism is guaranteed by the inductive hypothesis. If the minimum is attained by $$\hat p_k(s,a)^T \tilde V_{t+1}^{\pi_k} + \min(\rng{\tilde V_{t+1}^{\pi_k}} + \phi^+,H-t)\phi(s,a))$$ then two cases are possible.
\paragraph{Case I} It holds that
$$
\rng{\tilde V_{t+1}^{\pi_k}} + \phi^+ \geq H-t
$$
so that the bonus becomes $\hat p_k(s,a)^T \tilde V_{t+1}^{\pi_k} + (H-t)\phi(s,a)$ and
the procedure gives an identical result as the original \ubev in \cite{Dann17} and optimism is ensured.
\paragraph{Case II}
It holds that
$$\rng{\tilde V_{t+1}^{\pi_k}} + \phi^+ < H-t.$$
However in such case lemma \ref{ubevs:rangeok} can be applied (we are outside of the failure event and $\tilde V_{t+1}^{\pi_k}(s) \geq V_{t+1}^{*}(s)$ by the inductive hypothesis) to ensure $\rng{\tilde V_{t+1}^{\pi_k}} + \phi^+ \geq \rng{ V_{t+1}^{*}}$. This immediately implies that:
\begin{align*}
& \hat p_k(s,a)^T \tilde V_{t+1}^{\pi_k} + \( \rng{\tilde V_{t+1}^{\pi_k}} + \phi^+\)\phi(s,a)  \\
& \geq \hat p_k(s,a)^T V_{t+1}^{*} + \rng{V_{t+1}^{*}}\phi(s,a) \\
& \geq p(s,a)^T  V_{t+1}^{*}.
\end{align*}
where the last inequality follows from being outside of the failure event (in particular, outside of $F^V_k$). This holds for every state and action for a given timestep, proving the inductive step and guaranteeing optimism.
\end{proof}

\subsection{Regret Bounds of \ubevs on Episodic Stationary MDPs}
\label{ubevs:StationaryResult}
We now derive a high probability worst case regret upper bound when \ubevs is run on a stationary episodic MDP.
\begin{theorem}[\ubevs Regret]
\label{thm:ubevs_regret}	 With probability at least $1-\delta$ the regret of \ubevs is upper bounded by 
\begin{equation} 
	\tilde O(H\sqrt{SAT} + S^2AH^{2} + S\sqrt{S}AH^3)	
\end{equation}
jointly for all timesteps $T$.
\end{theorem}
\begin{proof}
	Outside the failure event \ubevs is optimistic (lemma \ref{ubevs:OptimisticPlanning}) which justifies the first passage below (the expansion in the last equality is standard, see for example the derivation of the main result in \cite{Dann17}):
	\begin{align*}
& \textsc{Regret}(K) \stackrel{def}{=}  \sum_{k = 1}^K V^{\pi^*}_1(s) - V^{\pi_k}_1(s) \\
& \stackrel{Optimism}{\leq} \sum_{k = 1}^K \tilde V^{\pi_k}_1(s) - V^{\pi_k}_1(s) = \sum_{k=1}^K \sum_{t \in [H]} \sum_{(s,a) \in L_{k}} w_{tk}(s,a) \Biggm( \underbrace{\(\tilde r(s,a) - r(s,a) \)}_{\text{Reward Estimation and Optimism}} + \\
& + \underbrace{\(\tilde p(s,a) - \hat p(s,a) \)^\top \tilde V_{t+1}^{\pi_k}}_{\text{Transition Dynamics Optimism}} +  \underbrace{\(\hat p(s,a) - p(s,a) \)^\top V_{t+1}^{*}}_{\text{Transition Dynamics Estimation}} +  \underbrace{\(\hat p(s,a) - p(s,a) \)^\top \( \tilde V_{t+1}^{\pi_k} - V^*_{t+1} \)}_{\text{Lower Order Term}}  \Biggm) + \sum_{k=1}^K \sum_{t \in [H]} \sum_{(s,a) \not \in L_{k}} w_{tk}(s,a)H \\
\numberthis{\label{eqn:AdvancedRegretDecompositionWithPseudoregret}}
\end{align*}
Corollary \ref{lem:MinimalContribution} ensures $\sum_{k=1}^K \sum_{t \in [H]} \sum_{(s,a) \not \in L_{k}} w_{tk}(s,a)H = \tilde O(SAH^2)$; the theorem is then proved by invoking lemmata \ref{ubevs:vstarterm_generic},\ref{ubevs:tildevterm_generic},\ref{ubevs:lowterm_generic}.
\end{proof}

\begin{lemma}
\label{ubevs:vstarterm_generic}
Outside the failure event for \ubevs it holds that:
\begin{equation*}
\sum_{k} \sum_{t\in [H]} \sum_{(s,a) \in L_k} w_{tk}(s,a) \left|(\hat P_k -  P)(s,a,t) V_{t+1}^{*} \right| = \tilde O\(H\sqrt{SAT} \)
\end{equation*}
\end{lemma}
\begin{proof}
The following inequalities hold true up to a constant:
\begin{align*}
\sum_{k} \sum_{t\in [H]} \sum_{(s,a) \in L_k} w_{tk}(s,a) \left|(\hat P_k -  P)(s,a,t) V_{t+1}^{*} \right| & \stackrel{a}{\lesssim} H \sum_{k} \sum_{t\in [H]} \sum_{(s,a) \in L_k} w_{tk}(s,a) \sqrt{\frac{2\llnp(n_{k}(s,a)) + \ln(\frac{27SA}{\delta})}{n_k(s,a)}} \\
& \stackrel{b}{\lesssim} H \sum_{k} \sum_{t\in [H]} \sum_{(s,a) \in L_k} w_{tk}(s,a) \sqrt{\frac{1}{n_k(s,a)}} \polylog(\cdot) \\
& \stackrel{c}{=} \tilde O\(H\sqrt{SAT}\)
\end{align*}

\begin{enumerate}[label=(\alph*)]
\item using the definition of failure event (in particular of $F_k^V$) and that $\rng V^*_t \leq H, \; \forall t$ 
\item since $n_k(s,a) \leq T$
\item using lemma \ref{lem:wsqrtn}.
\end{enumerate}
\end{proof}

\begin{lemma}
\label{ubevs:tildevterm_generic}
Outside the failure event for \ubevs it holds that:
\begin{equation*}
\sum_{k} \sum_{t\in [H]} \sum_{(s,a) \in L_k} w_{tk}(s,a) \( \left|(\tilde r_k - r)(s,a,t)\right| + \left|(\tilde P_k - \hat P_k)(s,a)\tilde V_{t+1}^{\pi_k} \right| \) \leq \tilde O \( H\sqrt{SAT} \).
\end{equation*}
\end{lemma}
\begin{proof}
Let $(\tilde s_{k},\tilde t_k) = \argmax_{s,t} \phi_k(s,\pi_k(s,t)) = \argmin_{s,t} n_k(s,\pi_k(s,t))$.
\begin{align*}
& \sum_{k} \sum_{t\in [H]} \sum_{(s,a)  \in L_k} w_{tk}(s,a) \( \left|(\tilde r_k - r)(s,a,t)\right| + \left|(\tilde P_k - \hat P_k)(s,a)\tilde V_{t+1}^{\pi_k} \right| \) \\
& \stackrel{a}{\lesssim}\sum_{k} \sum_{t\in [H]} \sum_{(s,a)  \in L_k } w_{tk}(s,a) \phi_k\(s,\pi_k(s,t)\) \( 1 + \min \{\rng \tilde V_{t+1}^{\pi_k}  + \phi^+,H\} \) \\
& \stackrel{b}{\lesssim} \sum_{k} \sum_{t\in [H]} \sum_{(s,a)  \in L_k} w_{tk}(s,a) \phi_k\(s,\pi_k(s,t)\) H\\
& \stackrel{c}{\lesssim}\sum_{k} \sum_{t\in [H]} \sum_{(s,a)  \in L_k} w_{tk}(s,a) H \sqrt{\frac{1}{n_k(s,a)}} \polylog(\cdot) = \tilde O\(H\sqrt{SAT}\)
\end{align*}

\begin{enumerate}[label=(\alph*)]
\item using the definition of failure event (in particular of $F_k^R$) and of exploration bonus
\item is using the crude upper bound $H$
\item is using lemma \ref{lem:wsqrtn}
\end{enumerate}
\end{proof}

\begin{lemma}
\label{ubevs:lowterm_generic}
Outside the failure event for \ubevs it holds that:
\begin{equation*}
\sum_{k} \sum_{t\in [H]} \sum_{(s,a) \in L_{k}} w_{tk}(s,a)  \left|(\hat p_k - p)(s,a,t) (\tilde V_{t+1}^{\pi_k} - V_{t+1}^{*}) \right| = \tilde O\(S^2AH^{2} + S\sqrt{S}AH^3 \)
\end{equation*}
\end{lemma}
\begin{proof}
We can write the following sequence of upper bounds: 
\begin{align*}
& \sum_{k} \sum_{t\in [H]} \sum_{(s,a) \in L_{k}} w_{tk}(s,a)  \left|(\hat p_k - p)(s,a,t) (\tilde V_{t+1}^{\pi_k} -  V_{t+1}^{*}) \right|  \lesssim \\
& \stackrel{(a)}{\lesssim} \sum_{k} \sum_{t\in [H]} \sum_{(s,a) \in L_{k}} w_{tk}(s,a) \sum_{s_{t+1}} \( \sqrt{\frac{p(s_{t+1}\mid s,\pi_k(s,t))}{n_k(s,\pi_k(s,t))}} + \frac{1}{n_k(s,a)} \) \(\tilde V_{t+1}^{\pi_k}(s_{t+1}) - V_{t+1}^{*}(s_{t+1}) \) \polylog \\
& \stackrel{(b)}{\lesssim} \sum_{k} \sum_{t\in [H]} \sum_{(s,a) \in L_{k}} w_{tk}(s,a) \sqrt{S} \sqrt{\frac{\sum_{s_{t+1}} p(s_{t+1}\mid s,\pi_k(s,t))(\tilde V_{t+1}^{\pi_k}(s_{t+1}) - V_{t+1}^{*}(s_{t+1}))^2}{n_k(s,\pi_k(s,t))}}\polylog \\
& + SH\sum_{k} \sum_{t\in [H]} \sum_{(s,a) \in L_{k}} \frac{w_{tk}(s,\pi_k(s,t))}{n_k(s,\pi_k(s,t))} \polylog \\
& \stackrel{(c)}{\lesssim} \sqrt{S} \sqrt{ \sum_{k} \sum_{t\in [H]} \sum_{(s,a) \in L_{k}} \frac{w_{tk}(s,a)}{n_k(s,a)}}\times  \( \sqrt{\sum_{k} \sum_{t\in [H]} \sum_{(s,a) \in L_{k}} w_{tk}(s,a) \sum_{s_{t+1}} p(s_{t+1}\mid s,\pi_k(s,t))(\tilde V_{t+1}^{\pi_k}(s_{t+1}) - V_{t+1}^{*}(s_{t+1}))^2}\)  \\
& + \tilde O(S^2AH) \\
& \stackrel{(d)}{\lesssim} \tilde O\(S\sqrt{A} \) \times  \( \sqrt{\sum_{k} \sum_{t\in [H]} \sum_{(s,a) \in L_{k}} w_{tk}(s,a) \sum_{s_{t+1}} p(s_{t+1}\mid s,\pi_k(s,t))(\tilde V_{t+1}^{\pi_k}(s_{t+1}) - V_{t+1}^{*}(s_{t+1}))^2}\)  + \tilde O(S^2AH) \\
& \stackrel{(e)}{\lesssim} \tilde O\(S\sqrt{A} \) \times  \( \sqrt{\sum_{k} \sum_{t \in [H]} \E_{s_{t+1} \sim \tilde \pi_k}(\tilde V_{t+1}^{\pi_k}(s_{t+1}) - V_{t+1}^{*}(s_{t+1}))^2}\)  + \tilde O(S^2AH) \\
& \stackrel{(f)}{\lesssim} \tilde O\(S\sqrt{A} \) \times \sqrt{\sum_{k} \sum_{t \in [H]} \E_{s_{t+1} \sim \tilde \pi_k}\(\sum_{\tau = 
t+1,\dots, H} \E_{(s_\tau,a_\tau)\sim \tilde \pi_k \mid s_{t+1}}  \min \{\frac{\sqrt{S}H}{\sqrt{n_k(s_\tau,a_\tau)}},H\} \times \polylog\)^2}  + \tilde O(S^2AH) \\
& \stackrel{(g)}{\lesssim} \tilde O\(S\sqrt{A} \) \times \sqrt{H\sum_{k} \sum_{t \in [H]} \E_{s_{t+1} \sim \tilde \pi_k}\sum_{\tau = 
t+1,\dots, H} \( \E_{(s_\tau,a_\tau)\sim \tilde \pi_k \mid s_{t+1}}  \min \{\frac{\sqrt{S}H}{\sqrt{n_k(s_\tau,a_\tau)}},H\} \times \polylog\)^2}  + \tilde O(S^2AH) \\
& \stackrel{(h)}{\lesssim} \tilde O\(S\sqrt{A} \) \times \sqrt{H\sum_{k} \sum_{t \in [H]} \E_{s_{t+1} \sim \tilde \pi_k}\sum_{\tau = 
t+1,\dots, H} \E_{(s_\tau,a_\tau)\sim \tilde \pi_k \mid s_{t+1}}  \min \{\frac{\sqrt{S}H}{\sqrt{n_k(s_\tau,a_\tau)}},H\}^2}   + \tilde O(S^2AH) \\
& \stackrel{(i)}{\lesssim} \tilde O\(S\sqrt{A} \) \times \sqrt{H\sum_{k} \sum_{t\in [H]} \sum_{\tau = 
t+1,\dots, H} \E_{(s_\tau,a_\tau)\sim \tilde \pi_k}  \min \{\frac{\sqrt{S}H}{\sqrt{n_k(s_\tau,a_\tau)}},H\}^2}   + \tilde O(S^2AH) \\
& \stackrel{(j)}{\lesssim} \tilde O\(S\sqrt{A} \) \times \sqrt{H^2\sum_{k}  \sum_{\tau = 
1,\dots, H} \E_{(s_\tau,a_\tau)\sim \tilde \pi_k}  \min \{\frac{\sqrt{S}H}{\sqrt{n_k(s_\tau,a_\tau)}},H\}^2}   + \tilde O(S^2AH) \\
& \stackrel{(k)}{\lesssim} \tilde O\(S\sqrt{A} \) \times \sqrt{H^2\sum_{k}  \sum_{\tau = 
1,\dots, H} \( \E_{\substack{(s_\tau,a_\tau)\sim \tilde \pi_k \\ (s_\tau,a_\tau) \in L_{k}}}  \min \{\frac{\sqrt{S}H}{\sqrt{n_k(s_\tau,a_\tau)}},H\}^2 + \E_{\substack{(s_\tau,a_\tau)\sim \tilde \pi_k \\ (s_\tau,a_\tau) \not \in L_{k}}}  \min \{\frac{\sqrt{S}H}{\sqrt{n_k(s_\tau,a_\tau)}},H\}^2 \)}   + \tilde O(S^2AH) \\
& \stackrel{(l)}{\lesssim} \tilde O\(S\sqrt{A} \) \times \sqrt{H^2\sum_{k}  \sum_{\tau = 
1,\dots, H} \( \E_{\substack{(s_\tau,a_\tau)\sim \tilde \pi_k \\ (s_\tau,a_\tau) \in L_{k}}}  \frac{SH^2}{n_k(s_\tau,a_\tau)} + \E_{\substack{(s_\tau,a_\tau)\sim \tilde \pi_k \\ (s_\tau,a_\tau) \not \in L_{k}}}H^2 \)}   + \tilde O(S^2AH) \\
& \stackrel{(m)}{\lesssim} \tilde O\(S\sqrt{A} \) \times \sqrt{H^2\sum_{k}  \sum_{\tau = 
1,\dots, H} \( SH^2 \times  \sum_{(s_\tau, a_\tau) \in L_{k}}  \frac{w_{\tau k}(s_\tau,a_\tau)}{n_k(s_\tau,a_\tau)} + H^2 \times \sum_{(s_\tau,a_\tau) \not \in L_{k}} w_{\tau k}(s_\tau,a_\tau) \)}   + \tilde O(S^2AH) \\
& \stackrel{(n)}{\lesssim} \tilde O\(S\sqrt{A} \) \times \sqrt{H^2 \times (SH^2 \times SA + H^2 \times SAH^2)}   + \tilde O(S^2AH) \\
& \stackrel{}{=} \tilde O\(S^2AH^{2} + S\sqrt{S}AH^3 + S^2AH \) = \tilde O\(S^2AH^{2} + S\sqrt{S}AH^3 \) \\
\end{align*}
\begin{enumerate}[label=(\alph*)]
	\item holds since we are outside of the failure event, and in particular of $F^P$ (see \cite{Dann17} for the definition)
	\item is using Cauchy-Schwartz on the first term and $\(\tilde V_{t+1}^{\pi_k}(s') - V_{t+1}^{*}(s') \) \leq H$ on the second
	\item is again Cauchy-Schwartz on the first term and lemma \ref{lem:wn} on the second
	\item is using lemma \ref{lem:wn} on the first term
	\item by definition of conditional expectation
	\item is using lemma \ref{lem:deltaoptimism}
	\item is using $(a_1+\dots+a_n)^2 \leq n\times (a^2_1+\dots+a^2_n)$
	\item is Jensen's inequality
	\item is the definition of conditional expectation
	\item makes the inner summation start at $\tau = 1$ instead of $\tau = t+1$ and the summation over $t$ gives the additional $H$ factor
	\item is splitting the expectation over states in $L_{k}$ and not in $L_{k}$ (this passage is an equality)
	\item chooses one of the maxes for each term
	\item re-expresses the expectation using the visitation probabilities $w_{\tau k}$
	\item makes use of lemma \ref{lem:wn} and lemma 
 \ref{lem:MinimalContribution}
\end{enumerate}
\end{proof}

\begin{lemma}
\label{lem:deltaoptimism}
	Outside of the failure event for \ubevs it holds that
\begin{equation}
  \tilde V^{\pi_k}_t(s_t) - V^{\pi_k}_t(s_t) \leq \sum_{\tau = 
t,\dots, H} \E_{(s_\tau,a_\tau)\sim \tilde \pi_k \mid s_t} \min \{ \frac{\sqrt{S}H}{\sqrt{n_k(s_\tau,a_\tau)}} , H \}  \times \polylog\end{equation}

\end{lemma}
\begin{proof}
\begin{align}
& \tilde V^{\pi_k}_{t}(s_\tau) - V^{\pi_k}_t(s_t) = \\
& = \sum_{\tau = 
t,\dots H}  \E_{(s_\tau,a_\tau)\sim \tilde \pi_k \mid s_t}  \min \Biggm\{  \underbrace{\(\tilde r(s_\tau,a_\tau) - r(s_\tau,a_\tau) \)}_{\text{Reward Estimation and Optimism}} + \underbrace{\(\tilde p(s_\tau,a_\tau) - \hat p(s_\tau,a_\tau) \)^\top \tilde V_{\tau+1}^{\pi_k}}_{\text{Transition Dynamics Optimism}} +  \underbrace{\(\hat p(s_\tau,a_\tau) - p(s_\tau,a_\tau) \)^\top V_{\tau+1}^{*}}_{\text{Transition Dynamics Estimation}} +  \\
& \underbrace{\(\hat p(s_\tau,a_\tau) - p(s_\tau,a_\tau) \)^\top \( \tilde V_{\tau+1}^{\pi_k} - V^*_{\tau+1} \)}_{\text{Lower Order Term}}, H \Biggm\} \times \polylog \\
& \leq \sum_{\tau = 
t,\dots, H} \E_{(s_\tau,a_\tau)\sim \tilde \pi_k \mid s_t} \min \{ \frac{\sqrt{S}H}{\sqrt{n_k(s_\tau,a_\tau)}} , H \}  \times \polylog
\end{align}
In the above expression we explicitly wrote the minimum $\min\{ \cdot, H\}$ since the estimation error and the bonus cannot exceed $H$ in each of the $(s_\tau,a_\tau)$ pairs. Here the dominant term is the ``Lower Order Term'' which we bound trivially using the fact that we are outside the event $F^{L1}$ (see \cite{Dann17}) and the value function is always bounded by $H$.
\end{proof}

\section{Regret Bounds for \ubevs on \MC{}}
\subsection{The Good Episodes on \MC{}}
In \cite{Dann17} the authors use the notion of \emph{nice} episodes. The goal is ensure that $\sum_{k \in [K]}\sum_t w_{tk}(s,a) \approx n_{K}(s,a)$ holds. We will do the same here directly using the regret framework and the properties of \MC{}. To this aim we define the \emph{Good Episodes} as follows:
\begin{definition}
\label{ubevs:GoodEpisodes}
On \MC{} an episode $k$ is good if the failure event does not occur and
\begin{equation}
n_k(s,\pi_k(s,t)) \geq \frac{1}{4} \sum_{i< k}\sum_{\tau\in [H]}w_{\tau i}(s,\pi_k(s,t))
\end{equation}
holds for all state $s$ and timesteps $t$.
\end{definition}

\begin{proposition}
\label{ubevs:mynice}
If $k$ is a good episode it holds that
\begin{equation}
n_k(s,\pi_k(s,t)) \geq \frac{1}{8} \sum_{i\leq k}\sum_{\tau \in [H]} w_{\tau i}(s,\pi_k(s,t))
\end{equation}
for all states $s$ and timesteps $t \in [H]$. 
\end{proposition}
\emph{Remark:} this differs from the definition because the summation on the right hand side \emph{includes} $k$.
\begin{proof}
Directly by the definition of being outside the failure event $F^N_k$:
$$
n_k(s,a) \geq \frac{1}{2} \sum_{i < k} \sum_{t \in [H]} w_{ti}(s,a) - H\ln \frac{9SA}{\delta}
$$
for every state $s$ and action $a$. This can be satisfied under good episodes only if
$$
\frac{1}{4} \sum_{i < k} \sum_{t \in [H]} w_{ti}(s,a) \geq H \ln \frac{9SA}{\delta} \geq 2H
$$
holds true.
This implies 
$$
\frac{1}{8} \sum_{i < k} \sum_{t \in [H]} w_{ti}(s,a) \geq H \geq \sum_{t \in [H]} w_{tk}(s,a).
$$
and finally
$$
n_k(s,a) \geq \frac{1}{4} \sum_{i < k} \sum_{t \in [H]} w_{ti}(s,a) \geq \frac{1}{8} \sum_{i < k} \sum_{t \in [H]} w_{ti}(s,a) + \sum_{t \in [H]} w_{tk}(s,a) \geq \frac{1}{8} \sum_{i \leq k} \sum_{t \in [H]} w_{ti}(s,a) 
$$
which is the statement.
\end{proof}

Next we prove a bound on the number of non-good episodes.
\begin{lemma}[Number of Non-Good Episodes]
\label{ubevs:NumNonGoodEpi}
Outside of the failure event \ubevs can have at most:
$$
\tilde O \(\frac{SAH}{\mu_{min}} \)
$$
non-good episodes if run on \MC{}.
\end{lemma}
\begin{proof}
If the episode is non-good and the failure event does not occur then:
\begin{equation}
n_k(s,\pi_k(s,t)) < \frac{1}{4} \sum_{i< k}\sum_{\tau\in [H]}w_{\tau i}(s,\pi_k(s,t))
\end{equation}
must hold.
However, since the failure event does not occur then we can use the definition of being outside the failure event $F^N_k$:
\begin{equation}
n_k(s,a) \geq \frac{1}{2} \sum_{i < k} \sum_{t \in [H]} w_{ti}(s,a) - H\ln \frac{9SA}{\delta}.
\end{equation}
Together they imply:
\begin{equation}
\frac{1}{4} \sum_{i< k}\sum_{\tau\in [H]}w_{\tau i}(s,\pi_k(s,t)) \geq n_k(s,\pi_k(s,t)) \geq \frac{1}{2} \sum_{i < k} \sum_{t \in [H]} w_{ti}(s,\pi_k(s,t)) - H\ln \frac{9SA}{\delta}.
\end{equation}
and so 
\begin{equation}
\label{ubevs:NonGoodCondition}
\frac{1}{4} \sum_{i< k}\sum_{\tau\in [H]}w_{\tau i}(s,\pi_k(s,t)) \leq H\ln \frac{9SA}{\delta}.
\end{equation}
must be true during non-good episodes.
Each time \ubevs chooses action $a=\pi_k(s,t)$ in state $s$ we must have $w_{tk}(s,\pi_k(s,t)) \geq \mu_{min}$. Therefore, equation \ref{ubevs:NonGoodCondition} can occur $\frac{4H}{\mu_{min}}\ln\frac{9SA}{\delta}$ episodes for a given $(s,a)$ pair. Since there are at most $S \times A$ pairs of states and actions we can have at most 
$$
\tilde O \(\frac{SAH}{\mu_{min}}\)
$$
non-good episodes.
\end{proof}

\begin{lemma}[Regret under Non-Good Episodes]
\label{ubevs:RegretNonGoodEpi}
Outside of the failure event on \MC{} \ubevs can have a regret of at most:
$$
\tilde O \(\frac{SAH^2}{\mu_{min}} \)
$$
due to non-good episodes.
\end{lemma}
\begin{proof}
Directly from the number of non-good episodes on \MC{} (lemma \ref{ubevs:NumNonGoodEpi}), each of which has a regret of at most $H$.
\end{proof}

\subsection{Regret Bounds of \ubevs on \MC{}}
Here we compute the regret of \ubevs when run on \MC{}.
\begin{theorem*}
\MainResult{}{}{}
\end{theorem*}
\begin{proof}
Using optimism of the algorithm we can write for any initial state $s$:
\begin{align}
\sum_{k=1}^{K}V^*_1(s) - V^{\pi_k}_1(s) \leq  \sum_{k=1}^{K}  \tilde V^{\pi_k}_1(s) - V^{\pi_k}_1(s) 
\end{align}
Next, by partitioning into the set of good episodes $G$ and those that are non-good we obtain:
\begin{align}
\leq  \sum_{k \in G}  \tilde V^{\pi_k}_1(s) - V^{\pi_k}_1(s) +   \sum_{k \not \in G}  \tilde V^{\pi_k}_1(s) - V^{\pi_k}_1(s) 
\end{align}
From lemma \ref{ubevs:NumNonGoodEpi} the regret due to non-good episodes is at most:
\begin{equation}
\tilde O \(\frac{SAH^2}{\mu_{min}} \).
\end{equation}
which is an upper bound on the regret induced by states not in $L_k$.
Thus it remains to bound the regret in good episodes $\sum_{k \in G}  \tilde V^{\pi_k}_1(s) - V^{\pi_k}_1(s)$ for states $(s,a)\in L_k$ which can be upper bounded by
\begin{align}
& \sum_{k \in G}\sum_{t\in [H]} \sum_{(s,a) \in L_k} w_{tk}(s,a) \(  \left|(\tilde r_k - r)(s,a,t)\right| + \left|(\tilde P_k - \hat P_k)(s,a,t)\tilde V_{t+1}^{\pi_k} \right|  \) + \\
& +  \sum_{k \in G} \sum_{t\in [H]} \sum_{(s,a) \in L_k} w_{tk}(s,a)  \( \left|(\hat P_k -  P)(s,a,t) V_{t+1}^{*} \right| + \left|(\hat P_k - P)(s,a,t) ( V_{t+1}^{*} - \tilde V_{t+1}^{\pi_k}) \right| \)
\end{align}
as explained for example in Lemma E.8 [Optimality Gap Bound On Friendly Episodes] of \cite{Dann17}.
The result then follows from combining lemma \ref{ubevs:vstarterm}, \ref{ubevs:tildevterm} and  \ref{ubevs:lowterm}, which bound each contribution outlined above.
The $\min$ between two results in the final regret bound follows by considering the minimum between the analysis here and the one for a generic MDP.
\end{proof}

\begin{lemma}
\label{ubevs:vstarterm}
The following bound holds true on \MC{}:
\begin{equation*}
\sum_{k \in G} \sum_{t\in [H]} \sum_{(s,a) \in L_k} w_{tk}(s,a) \left|(\hat P_k -  P)(s,a,t) V_{t+1}^{*} \right| = \tilde O\(\sqrt{SAT} \)
\end{equation*}
\end{lemma}
\begin{proof}
The following inequalities hold true up to a constant:
\begin{align*}
\sum_{k \in G} \sum_{t\in [H]} \sum_{(s,a) \in L_k} w_{tk}(s,a) \left|(\hat P_k -  P)(s,a,t) V_{t+1}^{*} \right| & \stackrel{a}{\lesssim} \sum_{k \in G} \sum_{t\in [H]} \sum_{(s,a) \in L_k} w_{tk}(s,a) \sqrt{\frac{2\llnp(n_{k}(s,a)) + \ln(\frac{27SA}{\delta})}{n_k(s,a)}} \\
& \stackrel{b}{\lesssim} \sum_{k \in G} \sum_{t\in [H]} \sum_{(s,a) \in L_k} w_{tk}(s,a) \sqrt{\frac{1}{n_k(s,a)}} \polylog(\cdot) \\
& \stackrel{c}{=} \tilde O\(\sqrt{SAT}\)
\end{align*}

\begin{enumerate}[label=(\alph*)]
\item using the definition of failure event (in particular of $F_k^V$) and that $\rng V^*_t \leq 1, \; \forall t$ for a contextual bandit problem as explained in the main text of this manuscript
\item since $n_k(s,a) \leq T$
\item using lemma \ref{lem:wsqrtn}.
\end{enumerate}
\end{proof}

In lemma \ref{main:RngVBound} in the main paper we show that $\rng \tilde V^{\pi_k}$ is upper bounded by $1$ plus a term that depends on the state with the lowest visit count. That is the key to obtain the result below:
\begin{lemma}
\label{ubevs:tildevterm}
On \MC{} it holds that:
\begin{equation*}
\sum_{k \in G} \sum_{t\in [H]} \sum_{(s,a) \in L_k} w_{tk}(s,a) \( \left|(\tilde r_k - r)(s,a,t)\right| + \left|(\tilde P_k - \hat P_k)(s,a)\tilde V_{t+1}^{\pi_k} \right| \) \leq \tilde O \( \sqrt{SAT} + H^2\sqrt{SH} \times \frac{SA}{\sqrt{\mu_{min}}} \).
\end{equation*}
\end{lemma}
\begin{proof}
Let $(\tilde s_{k},\tilde t_k) = \argmax_{s,t} \phi_k(s,\pi_k(s,t)) = \argmin_{s,t} n_k(s,\pi_k(s,t))$.
\begin{align*}
& \sum_{k \in G} \sum_{t\in [H]} \sum_{(s,a) \in L_k} w_{tk}(s,a) \( \left|(\tilde r_k - r)(s,a,t)\right| + \left|(\tilde P_k - \hat P_k)(s,a)\tilde V_{t+1}^{\pi_k} \right| \) \\
& \stackrel{a}{\lesssim}\sum_{k \in G} \sum_{t\in [H]} \sum_{(s,a) \in L_k} w_{tk}(s,a) \phi_k\(s,\pi_k(s,t)\) \( 1 + \rng \tilde V_{t+1}^{\pi_k}  + \phi^+\) \\
& \stackrel{b}{\lesssim} \sum_{k \in G} \sum_{t\in [H]} \sum_{(s,a) \in L_k} w_{tk}(s,a) \phi_k\(s,\pi_k(s,t)\) \( 1 + \( 1 + H^2 \sqrt{S} \phi \(\tilde s_{k},\pi_k(\tilde s_{k},\tilde t_k)\) \) \) \\
& \stackrel{c}{\lesssim}\sum_{k \in G} \sum_{t\in [H]} \sum_{(s,a) \in L_k} w_{tk}(s,a) \sqrt{\frac{1}{n_k(s,a)}}\( 1+1+ H^2 \sqrt{S} \sqrt{\frac{1}{n_{k}(\tilde s_k,\pi_k(\tilde s_k,\tilde t_k))}} \) \polylog(\cdot) \\
& \stackrel{d}{\lesssim} \sum_{k \in G} \sum_{t\in [H]} \sum_{(s,a) \in L_k} w_{tk}(s,a) \sqrt{\frac{1}{n_k(s,a)}}\polylog(\cdot) + H^2 \sqrt{S} \sum_{k \leq K} \sum_{t\in [H]} \sum_{(s,a) \in L_k} w_{tk}(s,a) \sqrt{\frac{1}{n_k(s,a)}}\sqrt{\frac{1}{n_{k}(\tilde s_k,\pi_k(\tilde s_k))}} \polylog(\cdot) \\
& \stackrel{e}{=} \tilde O\( \sqrt{SAT} + H^2\sqrt{SH} \times \frac{SA}{\sqrt{\mu_{min}}}  \).
\end{align*}

\begin{enumerate}[label=(\alph*)]
\item using the definition of failure event (in particular of $F_k^R$) and of exploration bonus
\item is using lemma \ref{main:RngVBound} in the main text
\item is using the definition of $\phi_k$ and the fact that $n_k(\cdot,\cdot) \leq T$ to put all the log terms into $\polylog(\cdot)$
\item is just splitting the two contributions
\item is using lemma \ref{lem:wsqrtn} and \ref{lem:wn}
\end{enumerate}
\end{proof}

\begin{lemma}
\label{ubevs:lowterm}
Outside the failure event on \MC{} it holds that:
\begin{equation*}
\sum_{k \in G} \sum_{t\in [H]} \sum_{(s,a) \in L_k} w_{tk}(s,a)  \left|(\hat p_k - p)(s,a,t) ( V_{t+1}^{*} - \tilde V_{t+1}^{\pi_k}) \right| = \tilde O \( \frac{\sqrt{H}S^2AH^2}{\sqrt{\mu_{min}}} \)
\end{equation*}
\end{lemma}

\begin{proof}
Using the definition $(\tilde s_k,\tilde t_k) = \max_{(s',t')} \phi_k(s',\pi_k(s',t'))$
\begin{align*}
& \sum_{k \in G} \sum_{t\in [H]} \sum_{(s,a) \in L_k} w_{tk}(s,a)  \left|(\hat p_k - p)(s,a,t) ( V_{t+1}^{*} - \tilde V_{t+1}^{\pi_k}) \right| \\
& \stackrel{a}{\lesssim} \sum_{k \leq K} \sum_{t\in [H]} \sum_{(s,a) \in L_k} w_{tk}(s,a) \|(\hat p_k - p) (s,a)\|_1 \| V_{t+1}^{*} - \tilde V_{t+1}^{\pi_k} \|_{\infty} \\
& \stackrel{b}{\lesssim} \sum_{k \in G} \sum_{t\in [H]} \sum_{(s,a) \in L_k} w_{tk}(s,a) \sqrt{\frac{S}{n_k(s,\pi_k(s,t))}} \| V_{t+1}^{*} - \tilde V_{t+1}^{\pi_k} \|_{\infty} \polylog(\cdot) \\
& \stackrel{c}{\lesssim} {S} H^2 \sum_{k \in G} \sum_{t\in [H]} \sum_{(s,a) \in L_k} w_{tk}(s,a) \sqrt{\frac{1}{n_k(s,\pi_k(s,t))}} \sqrt{\frac{1}{n_k(\tilde s_k,\pi_k(\tilde s_k,\tilde t_k))}} \polylog(\cdot) \\
& \stackrel{d}{=} {S}H^2 \times \tilde O \( \frac{SA\sqrt{H}}{\sqrt{\mu_{min}}} \)  \\
\end{align*}
\begin{enumerate}[label=(\alph*)]
\item is Holder's inequality
\item holds since we are under the good episodes which by definition are outside of the failure event and in particular outside of $(F_k^{L1})$ 
\item holds because $\tilde V_{t+1}^{\pi_k} - V_{t+1}^{*} \leq \tilde V_{t+1}^{\pi_k} - V_{t+1}^{\pi_k}$ (pointwise) and so we can use lemma \ref{ubevs:extrabonus}. 
\item by lemma \ref{lem:wn}
\end{enumerate}
\end{proof}

\subsection{Auxiliary Lemmas}
\begin{lemma}
\label{lem:wsqrtn}
The following holds true:
\begin{equation}
\sum_{k} \sum_{t\in [H]} \sum_{(s,a) \in L_k} w_{tk}(s,a) \sqrt{\frac{1}{n_k(s,a)}} \polylog(\cdot)  = \tilde O\( \sqrt{SAT} \)
\end{equation}
\end{lemma}
\begin{proof}
\begin{align*}
& \sum_{k} \sum_{t\in [H]} \sum_{(s,a) \in L_k} w_{tk}(s,a) \sqrt{\frac{1}{n_k(s,a)}} \polylog(\cdot) \\
& \stackrel{a}{\lesssim} \sqrt{\sum_{k} \sum_{t\in [H]} \sum_{(s,a) \in L_k} w_{tk}(s,a)} \sqrt{\sum_{k} \sum_{t\in [H]} \sum_{(s,a) \in L_k} w_{tk}(s,a)  \frac{1}{n_k(s,a)}} \polylog(\cdot) \\
& \stackrel{}{=} \tilde O \( \sqrt{SAT} \)
\end{align*}

\begin{enumerate}[label=(\alph*)]
\item by Cauchy-Schwartz
\item by a standard pigeonhole argument, see for example \cite{Dann17} for a sketch.
\end{enumerate}
\end{proof}

\begin{lemma}
\label{lem:wn}
The following holds true:
\begin{equation}
\sum_{\substack{k \in G}} \sum_{t\in [H]} \sum_{(s,a) \in L_k} w_{tk}(s,a) \sqrt{\frac{1}{n_k(s,a)}} \sqrt{\frac{1}{n_{k}(\tilde s_k,\pi_k(\tilde s_k,\tilde t_k))}} \polylog(\cdot)  = \tilde O \( \frac{SA\sqrt{H}}{\sqrt{\mu_{min}}} \) 
\end{equation}
\end{lemma}
\begin{proof}
\begin{align*}
& \sum_{\substack{k \in G}} \sum_{t\in [H]} \sum_{(s,a) \in L_k} w_{tk}(s,a) \sqrt{\frac{1}{n_k(s,a)}} \sqrt{\frac{1}{n_{k}(\tilde s_k,\pi_k(\tilde s_k,\tilde t_k))}} \polylog(\cdot)  \\ 
& \stackrel{a}{\lesssim} \sqrt{ \sum_{\substack{k \in G}} \sum_{t\in [H]} \sum_{(s,a) \in L_k}  \frac{w_{tk}(s,a) }{n_k(s,a)}} \sqrt{\sum_{\substack{k \in G}} \sum_{t\in [H]} \sum_{(s,a) \in L_k} \frac{w_{tk}(s,a)}{n_{k}(\tilde s_k,\pi_k(\tilde s_k,\tilde t_k))}} \polylog(\cdot) \\
& \stackrel{b}{\lesssim} \sqrt{SA} \sqrt{\sum_{\substack{k \in G}} \frac{H}{n_{k}(\tilde s_k,\pi_k(\tilde s_k,\tilde t_k))}} \polylog(\cdot) \\
& \stackrel{c}{\lesssim} \sqrt{SA} \sqrt{\sum_{\substack{k \in G}} \frac{H}{\sum_{i \leq k} \sum_{\tau \in [H]}w_{\tau i}(\tilde s_k,\pi_k(\tilde s_k,\tilde t_k))}} \polylog(\cdot) \\
& \stackrel{d}{\lesssim} \tilde O \( \frac{SA \sqrt{H}}{\sqrt{\mu_{min}}} \) \\
\numberthis{}
\end{align*}

\begin{enumerate}[label=(\alph*)]
\item by Cauchy-Schwartz
\item by a pigeonhole argument, see for example lemma E.5 of \cite{Dann17} and $\sum_{t\in [H]} \sum_{(s,a) \in L_k} w_{tk}(s,a) = H$
\item since we are in the good episodes (so using the definition of good episodes)
\item by lemma \ref{ubevs:onebyn}. The lemma is applied with the sequence $x_k = (\tilde s_k,\pi_k(\tilde s_k,\tilde t_k))$ which lives in $X = \mathcal S \times \mathcal A$; the function $a_i(x)$ is defined as $\sum_{\tau \in [H]}w_{\tau i}(x)$ and satisfies $a_i(x_i) \geq \mu_{min}$ by construction.
\end{enumerate}
\end{proof}

\begin{lemma}
\label{ubevs:onebyn}
Let $\{x\}_{i=1,2,\dots,K}$ be a sequence with $x_i \in X$ where $X$ is a set with cardinality $|X|$.
Let $\{a_i(x)\}_{i = 1,2,\dots,K}$ be a sequence of functions taking values $\geq 0$ and such that $a_i(x_i) \geq a_{min} > 0$. Then
\begin{equation}
\sum_{k=1}^K \frac{1}{\sum_{i \leq k}a_i(x_k)} = \tilde O \( \frac{|X|}{a_{min}}  \)
\end{equation}
holds.
\begin{proof}
Define the set:
\begin{equation}
\mathcal K_x = \{i\leq K:x_i = x\}
\end{equation}
Intuitively, this is ``the set of episodes where $x$ occurred".
Then the following sequence of inequalities holds true:
\begin{align*}
\sum_{k=1}^K \frac{1}{\sum_{i \leq k}a_i(x_k)} &  \stackrel{a}{=} \sum_x \sum_{k=1}^K \frac{\I{x = x_k}}{\sum_{i \leq k} a_i(x)}  \\
&  \stackrel{b}{=} \sum_x \sum_{k \in \mathcal{K}_x} \frac{1}{\sum_{i \leq k}a_{i}(x)}  \\
&  \stackrel{c}{\leq} \sum_x \sum_{k \in \mathcal{K}_x} \frac{1}{\sum_{i \in \mathcal K_x,i \leq k}a_{i}(x)}  \\
&  \stackrel{d}{\leq}\sum_x \sum_{k \in \mathcal{K}_x} \frac{1}{\sum_{i \in \mathcal K_x,i \leq k}a_{min}}\\
&  \stackrel{e}{\leq} \sum_x \sum_{k=1}^K \frac{1}{k a_{min}}\\
& = \sum_x \tilde O \( \frac{1}{a_{min}}\) = \tilde O \( \frac{|X|}{a_{min}}  \)
\end{align*}
\begin{enumerate}[label=(\alph*)]
\item holds since $\sum_x \I{x = x_k} = 1$
\item by definition of $\mathcal K_x$
\item holds by monotonicity since we are only adding the values for $a_i(x)$ if $x$ occurred in episode $i$, that is, if $x_i = x$
\item by definition, if $x$ occurred in episode $i$ then $a_i(x) = a_i(x_i)\geq a_{min}$
\item holds because by construction $\sum_{k \in \mathcal{K}_x} \frac{1}{\sum_{i \in \mathcal K_x,i \leq k}a_{min}} = \frac{1}{a_{min}}\( 1 + \frac{1}{2} + \frac{1}{3} + \dots + \frac{1}{|\mathcal K_x|} \) \leq \frac{1}{a_{min}} \( 1 + \frac{1}{2} + \frac{1}{3} + \dots + \frac{1}{K} \) = \frac{1}{a_{min}} \sum_{k = 1}^K \frac{1}{k}$
\end{enumerate}
\end{proof}
\end{lemma}

\section{\ucrl Appendix}
\label{ucrl:Appendix}
\subsection{Main Results}
Here we present our results for \ucrl. We make use of the notion of hitting time $T_{hit}$ which is the largest time to  transition between any two states under \emph{any} policy ($T^{M^*}_{hit}$ refers to the optimal policy). We prove a more general version than a reduction to MABs, in particular we use an MDP with maximum rewards achievable everywhere (certainly satisfied by MABs).
\propbandit

\begin{proof}in
We use the same notation as in the original \ucrl paper \cite{Jaksch10}. We show that:
\begin{enumerate}
\item the bias vector $\| \mathbf w_k\|_\infty$ is going to $\approx 0$ sufficiently fast so that:
\item the leading order term of the regret $\sum_k \mathbf v_k (\mathbf {\tilde P}_k - \mathbf{\mathbf P_k}) \mathbf w_k $ does not depend on $T$ except for a logarithmic factor. We assume that confidence intervals are not failing; failure of confidence intervals is addressed separately in the \ucrl paper \cite{Jaksch10}. 
\end{enumerate}

\paragraph{Bounding the bias vector}
We will be assuming throughout that the bias vector has been centered appropriately such that: $\| \mathbf w_k \|_{\infty} =  \max_{s} \mathbf w_k(s) -  \min_{s} \mathbf w_k(s)$ at every step $k$. This is easily obtained by forcing, for example, $\min_s \mathbf w_k(s) = 0$. We write $ V_{k,i}$ to indicate the value function vector at the beginning of the $i$-th iteration of extended value iteration in episode $k$ with $V_{k,0}$ the zero vector. Let $\overline s_k = \max_{s} \mathbf w_k(s)$ and $\underline s_k = \min_{s} \mathbf w_k(s)$. Clearly $\overline s_k$ and $\underline s_k$ are the maximizer and the minimizer, respectively, for the value function $V_{k,i}$ upon convergence (so when index $i$ is the last step of extended value iteration).  
Let $\tilde s_k = \argmax \tilde r(s,\pi_k(s_k)) $ be the state where the agent anticipates the highest optimistic rewards when following an $i$-step optimal policy $\pi_k$. The plan is to show that:

$$\| \mathbf w_k \|_{\infty} = \mathbf w(\overline s_k) - \mathbf w(\underline s_k) = V_{k,i}(\overline s_k) -  V_{k,i}(\underline s_k) \lesssim T_{hit}^{M^{\pi^*}}\sqrtucrlRsbar{\tilde s_k}.$$

This quantity may initially be bigger then the diameter $D$ provided in the \ucrl paper \cite{Jaksch10} but crucially it depends on the visit count to a state $\tilde s_k$ so it shrinks quickly if we visit such state often enough.  

To prove the bound we use an argument similar to that used in \cite{Jaksch10} to bound the value function with the diameter. Recall that by following any policy we can get to state $\overline s_k$ in at most  $T_{hit}^{M^{\pi^*}}$  steps in expectation. Also, $V_{k,i}(\underline s_k)$ is the total expected $i$-step reward of an optimal non-stationary $i$-step policy evaluated on the optimistic MDP starting from state $\underline s_k$. Since we are outside the failure event, the optimal policy $\pi^*$ and the true MDP is a feasible solution to extended value iteration. Now, if $V_{k,i}(\overline s_k) - V_{k,i}(\underline s_k) \gtrsim  T_{hit}^{M^{\pi^*}}\sqrtucrlRsbar{\tilde s_k} $ then an improved value for $V_{k,i}(\underline s_k)$ could be achieved by the following nonstationary policy: first follow $\pi^*$ (the true optimal policy) which takes at most $\ceil{T_{hit}^{M^{\pi^*}}}$ steps on average to get to $\overline s_k$. Then follow the optimal $i$-step policy from $\overline s_k$. At most $\ceil{T_{hit}^{M^{\pi^*}}}$ of the rewards of the policy from $\overline s_k$ are missed, but $\geq r^*$ is collected at every step up to $\overline s_k$ due to non-failing confidence intervals. The rewards missed are by assumption $\leq \max_{s}\tilde r(s,\pi_k(s)) = \tilde r(\tilde s, \pi_k(\tilde s))$. Together this implies that the agent loses at most $\ceil{T_{hit}^{M^{\pi^*}}}\(\tilde r(\tilde s, \pi_k(\tilde s)) - r^*\)$ before reaching $\overline s_k$. Thus:

\begin{align*}
\| \mathbf w_k \|_{\infty} & = \mathbf w(\overline s_k) - \mathbf w(\underline s_k) \\
& \stackrel{}{=} V_{k,i}(\overline s_k) -  V_{k,i}(\underline s_k) \\
&\stackrel{}{\leq} \ceil{T_{hit}^{M^{\pi^*}}}\(\tilde r(\tilde s, \pi_k(\tilde s)) - r^*\) \\
&\stackrel{a}{\leq} \ceil{T_{hit}^{M^{\pi^*}}}\(\hat r(\tilde s, \pi_k(\tilde s)) - r^* + \sqrtucrlRsbar{\tilde s_k}\) \\
&\stackrel{b}{\leq} \ceil{T_{hit}^{M^{\pi^*}}}\(\bar r(\tilde s, \pi_k(\tilde s)) - r^* + 2\sqrtucrlRsbar{\tilde s_k}\) \\
&\stackrel{c}{\leq} \ceil{T_{hit}^{M^{\pi^*}}}\( r^* - r^* + 2\sqrtucrlRsbar{\tilde s_k} \)\\
& = 2\ceil{T_{hit}^{M^{\pi^*}}}\sqrtucrlRsbar{\tilde s_k}. \numberthis \label{ucrl2wbandit}
\end{align*}

In the above inequalities:
\begin{enumerate}[label=(\alph*)]
\item follows from non-failing confidence interval for $(\tilde s,\pi_k(\tilde s_k))$, which allows us to go from the optimistic reward $\tilde r$ to the empirical estimate $\hat r$.
\item follows again from non-failing confidence interval for $(\tilde s,\pi_k(\tilde s_k))$. This time it allows us to go from the empirical reward $\hat r$ to the actual expected reward $\bar r$.
\item is using $\bar r(s,a) \leq r^*$.
\end{enumerate}

\paragraph{Bounding the Main Regret Term}
We now focus on the leading order term in the regret (see \cite{Jaksch10}):

\begin{align*}
\sum_k \mathbf v_k (\mathbf {\tilde P}_k - \mathbf{\mathbf P_k}) \mathbf w_k & \leq \sum_k \sum_{s,a} v_k(s,a) \left\| \mathbf {\tilde P}_k - \mathbf{\mathbf P_k} \right\|_1 \| \mathbf w_k\|_\infty \\
& \stackrel{a}{\lesssim} \sum_k \sum_{s,a} v_k(s,a) \sqrtucrlPsbar{s} * 2\ceil{T_{hit}^{M^{\pi^*}}}\sqrtucrlRsbar{\tilde s_k} \\
& \stackrel{}{\lesssim} \ceil{T_{hit}^{M^{\pi^*}}} \sum_k \sum_{s,a} v_k(s,a) \sqrtucrlPsbar{s} \sqrtucrlRsbar{\tilde s_k} \\
& \stackrel{}{\lesssim} T_{hit}^{M^{\pi^*}} \sqrt{S} \log\(\frac{2AT}{\delta}\) \sum_k \sum_{s,a} \sqrt{\frac{v_k(s,a)}{\max\{1,N_k(s,\pi_k(s)\}}} \sqrt{\frac{v_k(s,a)}{\max\{1,N_k(\tilde s_k,\pi_k(\tilde s_k)\}}} \\
& \stackrel{b}{\lesssim} T_{hit}^{M^{\pi^*}} \sqrt{S} \log\(\frac{2AT}{\delta}\) \sum_k \sum_{s,a}  \sqrt{\frac{v_k(s,a)}{\max\{1,N_k(\tilde s_k,\pi_k(\tilde s_k)\}}} \\
& \stackrel{c}{\lesssim} T_{hit}^{M^{\pi^*}} \sqrt{S} \log\(\frac{2AT}{\delta}\) \sqrt{\sum_k \sum_{s,a} 1}\sqrt{ \sum_k \sum_{s,a}  \frac{v_k(s,a)}{\max\{1,N_k(\tilde s_k,\pi_k(\tilde s_k)\}}}  \\
& \stackrel{d}{\lesssim} \tilde O\({\sqrt{S}SA T_{hit}^{M^{\pi^*}}}\) \sqrt{T_{hit}} \sqrt{ \sum_k \sum_{s,a}  \frac{v_k(s,a)}{T_{hit}\max\{1,N_k(\tilde s_k,\pi_k(\tilde s_k)\}}} \\
& \stackrel{e}{=} \tilde O\({\sqrt{S}SA T_{hit}^{M^{\pi^*}}\sqrt{T_{hit}} * SA}\) \quad \text{w.p.} > 1- \delta\\
& \stackrel{}{=} \tilde O\(\sqrt{S}(SA)^2 T_{hit}^{M^{\pi^*}}\sqrt{T_{hit}} \) \quad \text{w.p.} > 1- \delta\\
\end{align*}

In (a) we used the previously computed upper bound to the norm of the value function. In (b) we used that $\sqrt{\frac{v_k(s,a)}{\max\{1,N_k( s,\pi_k( s)\}}} \leq 1 $ since \ucrl starts a new episode once a counter for a certain $(s,a)$ pair is doubled. In (c) we used Cauchy-Schwartz and in (d) we bound the number of episodes by $\tilde O(SA)$ according to proposition 18 in \cite{Jaksch10}. We finally use lemma \ref{lem:contextualrate} in (e). This is the step where $T_{hit} < \infty$ is crucial because we are comparing the visits to $(s,a)$ to the counter for the past visits to a different state-action pair $N_k(\tilde s_k, \pi_k(\tilde s_k))$. Notice that the bound holds with high probability uniformly across all timesteps.

\paragraph{Bounding the Lower Order Regret Term}
We now bound the lower order term $\sum_k \mathbf v_k (\mathbf {P}_k - \mathbf{\mathbf I}) \mathbf w_k $.
Equation \ref{ucrl2wbandit} guarantees that the bias vector can be written as $\| \mathbf w_k\|_{\infty} \leq 2\ceil{T_{hit}^{M^{\pi^*}}}\sqrtucrlRsbar{\tilde s_k}$ which gives
$$
{\sum_{t=1}^T \| \mathbf w_{k(t)} \|_{\infty}^2} \leq  \tilde O \( S^2A^2 \( {T_{hit}^{M^{\pi^*}}} \)^2 T_{hit} \) \stackrel{def}{=} \tilde O (M)
$$
by lemma \ref{lem:wconvergence}. Finally lemma \ref{lem:martingaledifference} with $B = 0$ and the above definition for $M$ guarantees that outside the failure event:
$$\sum_k \mathbf v_k (\mathbf {P}_k - \mathbf{I}) \mathbf w_k = \tilde O \( \sqrt{S^2A^2\({T_{hit}^{M^{\pi^*}}}\)^2T_{hit}} + DSA\) = \tilde O \( SA{T_{hit}^{M^{\pi^*}}}\sqrt{T_{hit}} + DSA\) $$
holds true.
\paragraph{Summing up the Regret Contributions}
Together the bound obtained in the previous paragraphs and the bound for the rewards in \cite{Jaksch10}:
\begin{align*}
\sum_k \mathbf v_k (\mathbf{\tilde P_k} - \mathbf{P_k}) \mathbf w_k & = \tilde O\(\sqrt{S}(SA)^2 {T_{hit}^{M^{\pi^*}}}\sqrt{T_{hit}}\) \\
\sum_k \mathbf v_k (\mathbf {P}_k - \mathbf{I}) \mathbf w_k & = \tilde O \( SA{T_{hit}^{M^{\pi^*}}}\sqrt{T_{hit}} + DSA\) \\
\sum_k v_k(s,a) |\tilde r_k(s,a) - r(s,a)| & = \tilde O\(\sqrt{SAT}\)
\end{align*}
along with other lower order terms concludes our regret bound.
Finally union bound between the ``failure events'' considered in this analysis which has measure $o(\delta)$ and those considered in the original analysis of \ucrl, which also have measure $\delta$, concludes the proof.
\end{proof}

For completeness we examine what happens if the rewards on MABs are known exactly
\propbanditzeroreward

\begin{proof}
Consider running extended value iteration to compute the optimistic policy for \ucrl. For any optimistic transition probability matrix $\mathbf{\tilde{P}_k}$ it holds that $\mathbf{\tilde{P}_k} \1 = \1$ since $\1$ is a right eigenvector of any transition probability matrix. 
Now we show that extended value iteration as detailed in \cite{Jaksch10} must converge in two steps.
Let $ R_{\pi}$ be the reward vector induced by the agent's policy; we have that after the first step the value function is $\tilde V = \max_{\pi} \tilde R_{\pi}=  r^* \1$ since $\max_a \tilde r(s,a) = \max_a \bar r(s,a) = r^* \; \forall s$. After the second update the value function reads: $ \max_{\pi} R_{\pi} + \mathbf P_{\pi} \tilde V = \max_{\pi} R_{\pi} + r^*\1 = 2r^*\1$. Extended value iteration now has converged (see \cite{Jaksch10} for the termination conditions) finding the optimistic policy $\pi(s) = \argmax_a \tilde r(s,a) = \argmax_a \bar r(s,a) = \pi^*(s)\; \forall s$ and thus we have that the optimistic policy coincides with the optimal policy. This argument depends neither on the world dynamics nor on the data collected; since it can be applied at every episodes by induction we have that \ucrl always follows an optimal policy, achieving zero regret.
\end{proof}

Finally we examine what happens on \MC{}.
\propcontextualbandit

The proof idea is the following. Since there is a positive visitation frequency to every state, the agent can collect sufficient data in all states. The value function is not converging to uniform but it will be eventually bounded by a constant of order $1$ which is the maximum reward attainable in one step. This is similar to setting $D \approx 1$ in the original regret bound for \ucrl given in \cite{Jaksch10}. In short,  a value of $\approx 1$ quickly becomes an (over)estimate of the optimistic value function from the agent's viewpoint.

\begin{proof}
We use the same notation as in the original \ucrl paper \cite{Jaksch10}. We show that 1) the bias vector $\| \mathbf w_k\|_\infty$ is going to $\approx 1$ sufficiently fast so that 2) the leading order term of the regret $\sum_k \mathbf v_k (\mathbf {\tilde P}_k - \mathbf{\mathbf P_k}) \mathbf w_k $ is of order $S\sqrt{AT}$ plus terms that do not depend on $T$ except for a logarithmic factor. \\

\paragraph{Bounding the bias vector}
We now bound the bias vector $\| \mathbf w_k \|_{\infty} =  \max_{s} \mathbf w_k(s) -  \min_{s} \mathbf w_k(s)$ at step $k$. Let $\overline s_k = \max_{s} \mathbf w_k(s)$ and $\underline s_k = \min_{s} \mathbf w_k(s)$. Using equation (13) in \cite{Jaksch10} we know that upon termination of extended value iteration:

\begin{equation}
\label{maxw_contextual}
 \max_{s} \mathbf w_k(s)  = \mathbf w_k(\overline s_k) \leq \tilde r_k(\overline s_k, \pi_k(\overline s_k)) -\tilde \rho_k + \mathbf{\tilde P_k}(\overline s_k,\pi_k(\overline s_k))^T \mathbf w_k + \frac{1}{\sqrt{T_k}},
\end{equation}

where $\tilde \rho$ is the average per-step reward of the optimistic policy on the optimistic MDP.
The $\frac{1}{\sqrt{T_k}}$ term is a "planning error" which follows from prematurely stopping extended value iteration to save computations.

A lower bound on $\mathbf w_k(\underline s_k)$ is given below, where $ \mu $ denotes the true underlying distribution where the states are sampled from: 

\begin{align*}
 \min_{s} \mathbf w_k(s) & =  \mathbf w_k(\underline s_k) \\
 & \geq \tilde r_k(\underline s_k,\pi_k(\underline s_k)) - \tilde \rho_k + \mathbf{\tilde P_k}(\underline s_k,\pi_k(\underline s_k))^T \mathbf w_k - \frac{1}{\sqrt{T_k}} \\
& \stackrel{a}{\geq} \bar r_k(\underline s_k,\pi_k(\underline s_k)) - \tilde \rho_k  + \mu^T \mathbf w_k - \frac{1}{\sqrt{T_k}}\\
& \geq -\tilde \rho_k + \mu^T \mathbf w_k  - \frac{1}{\sqrt{T_k}}\\
\end{align*}

assuming non-failing confidence intervals and recalling that the rewards are all positive. Failure of confidence intervals is dealt separately in the \ucrl paper \cite{Jaksch10}. We observe that (a) follows from the fact that the agent is maximizing over the rewards and the transition model within their respective confidence intervals. 

Hence we have that the difference in the optimistic value function is:

\begin{align*}
\| \mathbf w_k \|_\infty & = \mathbf w_k(\overline s_k) - \mathbf w_k(\underline s_k) \\ 
& \leq \tilde r(\overline s_k, \pi_k(\overline s_k)) +  \( \mathbf{\tilde P_k}(\overline s_k,\pi_k(\overline s_k)) - \mu\)^T \mathbf w_k + \frac{2}{\sqrt{T_k}}\\
& \stackrel{a}{\leq} 1 +  \left\| \mathbf{\tilde P_k}(\overline s_k,\pi_k(\overline s_k)) -\mu \right \|_1 \| \mathbf  w_k \|_\infty + \frac{2}{\sqrt{T_k}} \\
& \stackrel{b}{\leq} 1 +  \sqrt{\frac{14S\log\frac{2AT_k}{\delta}}{2\max\{1,N_k(\overline s_k,\pi_k(\overline s_k))\}}}  \| \mathbf w_k \|_\infty + \frac{2}{\sqrt{T_k}} \\
& \stackrel{c}{\lesssim} 1 +  D\sqrt{\frac{14S\log\frac{2AT_k}{\delta}}{2\max\{1,N_k(\overline s_k,\pi_k(\overline s_k))\}}} \\
\numberthis \label{ucrl2wcontextualbandit}
\end{align*}

where $D$ is the diameter. Notice that we very crudely upper bounded $\frac{2}{\sqrt{T_k}} \leq 2$, while in fact it very rapidly decreases to zero. We have used Holder's inequality in a); in b) we used that the dynamics are the same everywhere and confidence intervals hold; finally in c) we bound the bias vector with the diameter.

\paragraph{Bounding the Main Regret Term}
The leading order term in the regret becomes:

\begin{align*}
\sum_k \mathbf v_k (\mathbf {\tilde P}_k - \mathbf{\mathbf P_k}) \mathbf w_k & \leq \sum_k \sum_{s,a} v_k(s,a) \left\| \mathbf {\tilde P}_k - \mathbf{\mathbf P_k} \right\|_1 \| \mathbf w_k\|_\infty \\
& \lesssim \sum_k \sum_{s,a} v_k(s,a) \sqrt{\frac{14S\log\frac{2AT_k}{\delta}}{2\max\{1,N_k(s,\tilde\pi_k(s))\}}} \( 1 +  D\sqrt{\frac{14S\log\frac{2AT_k}{2\delta}}{2\max\{1,N_k(\overline s_k,\tilde\pi_k(\overline s_k))\}}}  \) \\
& \lesssim \sum_k \sum_{s,a} v_k(s,a) \sqrt{\frac{14S\log\frac{2AT_k}{\delta}}{\max\{1,N_k( s,\tilde\pi_k( s))\}}} \\
& +  D \sum_k \sum_{s,a} v_k(s,a) \sqrt{\frac{14S\log\frac{2AT_k}{\delta}}{\max\{1,N_k(s,\tilde\pi_k(s))\}} 
\frac{14S\log\frac{2AT_k}{\delta}}{\max\{1,N_k(\overline s_k,\tilde\pi_k(\overline s_k))\}}} \\
& \stackrel{a}{\lesssim} S\sqrt{AT} \log\frac{2AT}{\delta} +  28DS \sum_k \sum_{s,a}\log\frac{2AT_k}{\delta}  \sqrt{\frac{v_k(s,a)}{\max\{1,N_k( s,\tilde\pi_k( s))\}} 
\frac{v_k(s,a)}{\max\{1,N_k(\overline s_k,\tilde\pi_k(\overline s_k))\}}} \\
& \stackrel{b}{\lesssim} S\sqrt{AT} \log\frac{2AT}{\delta} +  28DS\log\frac{2AT}{\delta} \sum_k \sum_{s,a}\sqrt{ 
\frac{v_k(s,a)}{\max\{1,N_k(\overline s_k,\tilde\pi_k(\overline s_k))\}}} \\
& \stackrel{c}{\lesssim} S\sqrt{AT} \log\frac{2AT}{\delta} +  28DS\log\frac{2AT}{\delta} \sum_{s,a} \( \sqrt{\sum_k 1} \sqrt{ \sum_k
\frac{v_k(s,a)}{\max\{1,N_k(\overline s_k,\tilde\pi_k(\overline s_k))\}}} \) \\
& \stackrel{d}{\lesssim} S\sqrt{AT} \log\frac{2AT}{\delta} +  28DS\log\frac{2AT}{\delta} \sqrt{\mmax}  \sqrt{2T_{hit}} \sum_{s,a}  \(\sqrt{ \sum_k
\frac{v_k(s,a)}{{2T_{hit}} \max\{1,N_k(\overline s_k,\tilde\pi_k(\overline s_k))\}}} \) \\
& = \tilde O\( S\sqrt{AT} \)  + \tilde O \(DS^3A^2\sqrt{T_{hit}} \) \\
\end{align*}

with probability at least $1-\delta$ jointly for all time-steps $T \geq SA$. This expression is $ \tilde O(S\sqrt{AT})$ up to polylogarithmic terms and lower order terms.
In step (a) we used equation (20) in \cite{Jaksch10} to claim $\sum_{s,a}\sum_k \frac{v_k(s,a)}{\sqrt{\max \{1,N_k(s,a)\}}} \leq (\sqrt{2}+1) \sqrt{SAT}$. In (b)
we used the property that \ucrl terminates the episode when the counts for the visits to some $(s,a)$ pair doubles, that is, when $v_k(s,a) = N_k(s,a)$ so that 
$$
\sqrt{\frac{v_k(s,a)}{\max\{1,N_k(s,a)\}}} \leq 1.
$$
holds.
In (c) we used Cauchy-Schwartz inequality and in (d) the we bound the maximum number of episodes $m \leq \mmax$ according to Proposition 18 in \cite{Jaksch10} and also multiplied and divided by $\sqrt{2T_{hit}}$. In the final passage lemma \ref{lem:contextualrate} was used.
\paragraph{Bounding the Lower Order Regret Term}
We now bound the lower order term $\sum_k \mathbf v_k (\mathbf {P}_k - \mathbf{\mathbf I}) \mathbf w_k $.
Equation \ref{ucrl2wcontextualbandit} guarantees that the bias vector can be written as $\| \mathbf w_k\|_{\infty} \lesssim 1 +  D\sqrt{\frac{14S\log\frac{2AT_k}{\delta}}{2\max\{1,N_k(\overline s_k,\pi_k(\overline s_k))\}}}$ which gives:
$$
{\sum_{t=1}^T \| \mathbf w_{k(t)} \|_{\infty}^2} \lesssim  T+\tilde O (D^2S^3 A^2T_{hit}) \stackrel{def}{=} T+\tilde O (M)
$$
by lemma \ref{lem:wconvergence}. Finally lemma \ref{lem:martingaledifference} with $B = O(1)$ and the above definition for $M$ guarantees that outside the failure event:
$$\sum_k \mathbf v_k (\mathbf {P}_k - \mathbf{I}) \mathbf w_k = \tilde O\(\sqrt{T} + \sqrt{D^2S^3 A^2T_{hit}} + DSA\) = \tilde O\(\sqrt{T} + {DS^{\frac{3}{2}} A\sqrt{T_{hit}}} + DSA\) $$
holds true.
\paragraph{Summing up the Regret Contributions}
Together the bound obtained in the previous paragraph and the bound for the rewards in \cite{Jaksch10}:
\begin{align*}
\sum_k \mathbf v_k (\mathbf{\tilde P_k} - \mathbf{P_k}) \mathbf w_k & =  \tilde O\( S\sqrt{AT} \)  + \tilde O \(DS^3A^2\sqrt{T_{hit}} \) \\
\sum_k \mathbf v_k (\mathbf {P}_k - \mathbf{I}) \mathbf w_k & =  \tilde O\(\sqrt{T} + {DS^{\frac{3}{2}} A\sqrt{T_{hit}}} + DSA\) \\
\sum_k v_k(s,a) |\tilde r_k(s,a) - r(s,a)| & = \tilde O(\sqrt{SAT})
\end{align*}
along with other lower order terms concludes our regret bound.
Finally union bound between the ``failure events'' considered in this analysis which has measure $o(\delta)$ and those considered in the original analysis of \ucrl, which also have measure $\delta$, concludes the proof.

\end{proof}

\begin{lemma}
\label{lem:martingaledifference}
Consider running $\ucrl$ on an MDP with finite maximum mean hitting time $T_{hit}$ and let $\pi_k(\cdot)$ the policy followed during the $k$-th step. If at every timesteps it holds that:
\begin{equation}
\sum_{j=1}^t \| \mathbf w_{k(j)} \|_{\infty}^2 \leq B^2t + \tilde O (M)
\end{equation}
for some constants $B,M$ then
$$ 
\sum_k \mathbf v_k (\mathbf {P}_k - \mathbf{I}) \mathbf w_k = \tilde O (B\sqrt{T} + \sqrt{M} + DSA)
$$
holds true with probability at least $1- o(\delta)$ jointly for all timesteps $t$.
\end{lemma}
\begin{proof}
Follow the same step as in \cite{Jaksch10} in paragraph 4.3.2 (the true transition matrix). We define $X_t \stackrel{def}{=} \( p(\cdot | s_t,a_t )-e_{s_t}\) \mathbf w_{k(t)}\1_{\text{conf} (t)}\1_{\text{w}(t)}$ where in particular $\1_{\text{conf}(t)}$, $\1_{\text{w}(t)}$ are the indicators for the event that the confidence intervals are not failing at tilmestep $t$ and that lemma \ref{lem:contextualrate} holds \emph{up to} time $t$, respectively. Here we cannot use Azuma-Hoeffding inequality because we do not have a deterministic bound on the $\| \mathbf w_k \|_{\infty}$'s which is at the same time stronger than $ \| \mathbf w_k \|_{\infty} \leq D$. To get around this notice that the above definition for the $X_t$'s guarantees that $X_t$ is still a sequence of martingale differences. To obtain a stronger bound than that in \cite{Jaksch10} we need to use \emph{Bernstein Inequality} which states that if $\text{Var}(\cdot)$ indicates the variance and $|X_t| \leq D$ is a martingale difference sequence the following statement holds true (see for example \cite{CL06} lemma A.8):
\begin{equation}
P \( \sum_{t=1}^T X_t \geq \epsilon\) \leq e^{-\frac{\epsilon^2}{2\sum_{t=1}^T \text{Var}(X_t)+2D\epsilon/3}}.
\label{Bernstein}
\end{equation}
A bound on the variance is given below:
\begin{align*}
\sum_{t=1}^T \text{Var}(X_t) & \leq \sum_{t=1}^T \E(X_t)^2 \\
& \leq \sum_{t=1}^T  \E ( \( p(\cdot | s_t,a_t )-e_{s_t}\)^T\mathbf w_{k(t)} )^2 \\
& \leq   \sum_{t=1}^{T} \| \(p(\cdot | s_t,a_t )-e_{s_t}\)^2\|_1 \| \mathbf w_{k(t)}\|_{\infty}^2 \\
& \simeq  \sum_{t=1}^{T} \| \mathbf w_{k(t)}\|_{\infty}^2 \leq B^2T +  \tilde O (M) 
\end{align*}
by hypothesis.
We have at most $T$ non-zero terms in the martingale sequence $\{X_t\}_{t = 1,\dots T}$. Now choose $\epsilon$ such that the right hand side of \ref{Bernstein} is $\leq \(\frac{\delta}{T}\)^{3}$ so that a further union bound over $T$ guarantees that the statement of the theorem holds with probability at least $1- o(\delta)$ uniformly across all timesteps. An $\epsilon = \tilde O (D+\sqrt{B^2T +M})$  suffices implying that together with the bound in \cite{Jaksch10} $\sum_k \mathbf v_k (\mathbf {P}_k - \mathbf{I}) \mathbf w_k \leq \sum_{t=1}^T X_t + \tilde O(DSA)$ and outside the failure event:
$$
\sum_k \mathbf v_k (\mathbf {P}_k - \mathbf{I}) \mathbf w_k = \tilde O (B\sqrt{T} + \sqrt{M} + DSA)
$$
holds as claimed.
\end{proof}

\subsection{Convergence of the Bias Vector and Bounds on the Visitation Ratio}

\begin{lemma}
\label{lem:wconvergence}
If for \textsc{ucrl2} at every episode k it holds that:
\begin{equation}
\| \mathbf w_k\|_{\infty} \leq B + C\sqrt{\frac{\log(2SAt_k/\delta)}{\max\{1,N_k(s_k,\pi (s_k)\}}}
\end{equation}
for some state $s_k$ and some constants $B,C$ on an MDP with $T_{hit} < \infty$ then if $k(t)$ is the episode that contains timestep $t$ the following two statements hold true with probability at least $1- o(\delta)$ jointly for all timesteps:
\begin{align*}
\sum_{t=1}^T \| \mathbf w_{k(t)} \|_{\infty}   &\leq  BT + C \tilde O (SA\sqrt{TT_{hit}}) \\
{\sum_{t=1}^T \| \mathbf w_{k(t)} \|_{\infty}^2} &\lesssim  B^2T + C^2 \tilde O (S^2A^2T_{hit})
\label{wbound}
\end{align*}
\end{lemma}
\begin{proof}
\begin{align*}
\sum_{t=1}^T \| \mathbf w_{k(t)} \|_{\infty} & \leq BT + \sum_{t=1}^T C\sqrt{\frac{\log(2SAt_k/\delta)}{\max\{1,N_{k(t)}(s_{k(t)},\pi (s_{k(t)})\}}} \\
& \leq BT + C\sqrt{T}\sqrt{ \sum_{t=1}^T  \frac{\log(2SAt_k/\delta)}{\max\{1,N_{k(t)}(s_{k(t)},\pi (s_{k(t)})\}}} \\
& \leq BT + C\sqrt{T}\log(2SAT/\delta)\sqrt{ \sum_{t=1}^T  \frac{1}{\max\{1,N_{k(t)}(s_{k(t)},\pi (s_{k(t)})\}}} \\
& \leq BT + C\sqrt{TT_{hit}}\log(2SAT/\delta)\sqrt{ \sum_{k}\sum_{s,a}  \frac{v_k(s,a)}{T_{hit}\max\{1,N_{k}(s_{k},\pi (s_{k})\}}} \\
& \leq BT + C \tilde O (SA\sqrt{TT_{hit}}) \\
\end{align*}
where we used that $(t_{k+1}-1) - (t_k) = \sum_{s,a}v_k(s,a)$  and lemma \ref{lem:contextualrate} in the final passage. Since it holds that
$$
\| \mathbf w_k\|_{\infty}^2 \lesssim B^2 + C^2{\frac{\log(2SAt_k/\delta)}{\max\{1,N_k(s_k,\pi (s_k)\}}}
$$
the second statement of the theorem is justified as follows: 
\begin{align*}
\sum_{t=1}^T \| \mathbf w_{k(t)} \|_{\infty}^2 & \lesssim B^2T + C^2\sum_{t=1}^T{\frac{\log(2SAt_k/\delta)}{\max\{1,N_{k(t)}(s_{k(t)},\pi (s_{k(t)})\}}} \\
& \leq BT + C^2\log(2SAT/\delta){ \sum_k\sum_{t=t_k}^{t_{k+1}-1}  \frac{1}{\max\{1,N_{k(t)}(s_{k(t)},\pi (s_{k(t)})\}}} \\
& \leq BT + C^2\log(2SAT/\delta)T_{hit}{ \sum_k \sum_{s,a}\frac{v_k(s,a)}{T_{hit}\max\{1,N_{k}(s_{k},\pi (s_{k})\}}} \\
& \leq BT + \tilde O (C^2T_{hit} S^2A^2) \\
\end{align*}
where again we used that $(t_{k+1}-1) - (t_k) = \sum_{s,a}v_k(s,a)$  and lemma \ref{lem:contextualrate} in the final passage.
\end{proof}

\begin{lemma}
\label{lem:contextualrate}
Consider running $\ucrl$ on an MDP with finite maximum mean hitting time $T_{hit}$, and let $s_{k_1},s_{k_2}$ be two states and $\pi_k(\cdot)$ the policy followed during the $k$-th step. 
If $T \geq SA$ then 
\begin{equation}
\sqrt{\sum_k\frac{v_k(s_{k_1},\pi(s_{k_1}))}{{2T_{hit}} \max \{ 1, N_k(s_{k_2},\pi(s_{k_2})) \} }} = \tilde O\(\sqrt{SA}\) 
 \label{lem:contextualratecondition}
 \end{equation} 
with probability at least $1- o(\delta)$ jointly for all timesteps.
\end{lemma}
\begin{proof}
We notice that $T \geq SA$ is required to use Proposition $18$ in \cite{Jaksch10} which bounds the number of episodes as a function of time. Let the number of episodes up to time $T$ be $m$. We have the following: 

\begin{align*}
& \sqrt{ \sum_k^m \frac{v_k(s_{k_1},\pi(s_{k_1}))}{2T_{hit}\max \{ 1, N_k(s_{k_2},\pi(s_{k_2})) \} }}  \\
& \stackrel{a}{\leq} \sqrt{m \E Y_k + \ceil*{\( c + \sqrt{c^2+4mc} \)} } \quad \text{w.p. > } 1-\delta \\
& \leq \tilde O\(\sqrt{SA}\) \quad \text{w.p. > } 1-\delta
\end{align*}

Before explaining (a), we mention that $\{ Y_k\}_{k = 1,...,m}$ are i.i.d geometric random variables with success probability $= \frac{1}{2}$ and thus $\E Y_k = 2 $. Furthermore, the constant $c = \log{\frac{6SAT^3}{\delta^2}}$ and according to Proposition 18 of \cite{Jaksch10} we can bound the number of episodes $m \leq \mmax$. The final expression depends on $T$ only via a polylog term and the proof is complete provided that we justify that (a) holds with probability at least $1- o(\delta)$ jointly for all timesteps $T$. 
The plan is to define a new sequence of i.i.d. geometric random variables $\{Y_k\}_{k=1,...,m}$ with success probability $p=\frac{1}{2}$ and use lemma \ref{stochasticdominance} to show that each of the $Y_k$'s first-order stochastically dominates the random variable ${\frac{v_k(s_{k_1},\pi(s_{k_1}))}{2T_{hit}\max \{ 1, N_k(s_{k_2},\pi(s_{k_2})) \} }}$ in the sense that 
$$P(Y_k > {x}) \geq P\({\frac{v_k(s_{k_1},\pi(s_{k_1})))}{2T_{hit}\max \{ 1, N_k(s_{k_2},\pi(s_{k_2}))) \} }} > x \), \forall x \in \mathbb{R}.$$
In other words, the $Y_i$'s "stochastically overestimate" the random variables ${\frac{v_k(s_{k_1},\pi(s_{k_1}))}{2T_{hit}\max \{ 1, N_k(s_{k_2},\pi(s_{k_2})) \} }}$ which are nor independent nor identically distributed. This is useful to simplify the problem. Since the bound in lemma \ref{stochasticdominance} holds regardless of the history, we can claim through Theorem 1.A.3 in \cite{stochasticorders} that a similar expression holds for the sum: 
$$P\(\sum_k^m Y_k > x\) \geq P\(\sum_k^m {\frac{v_k(s_{k_1},\pi(s_{k_1}))}{2T_{hit}\max \{ 1, N_k(s_{k_2},\pi(s_{k_2})) \} }} > x \), \forall x \in \mathbb{R}.$$
This justifies the first inequality below where we estimate the tail probability with a confidence interval chosen such that the final bound hold. In particular $c = \log{\frac{6SAT^3}{\delta^2}}$:

\begin{align*}
& P\( \sum_k^m \frac{v_k(s_{k_1},\pi(s_{k_1}))}{2T_{hit}\max \{ 1, N_k(s_{k_2},\pi(s_{k_2})) \} } > m\E Y_k + \ceil*{\( c + \sqrt{c^2+4mc} \)} \) \\
& < P\( \sum_k^m Y_k >  m\E Y_k + \ceil*{\( c + \sqrt{c^2+4mc} \)}  \) \\
& < e^{-c} \\
& = e^{-\log{\frac{6SAT^3}{\delta^2}}} \\
& = \frac{\delta^2}{6SAT^3}
\end{align*}

With the number of episodes crudely bounded by $m \leq T$, union bound over all possible values for $m$ and all possible timesteps $T$ along with $S$ states and $A$ actions yields that the tail probability is $o(\delta)$.
For the second inequality we have used lemma \ref{geometric_bound}. This is similar to Hoeffding inequality but modified for geometric random variables, which are not bounded. The lemma is applied with ${\epsilon} > \ceil*{\( c + \sqrt{c^2+4mc} \)} = \tilde O (\sqrt{SA})$. 
\end{proof}

\begin{lemma}
\label{stochasticdominance}
Consider running $\ucrl$ on an MDP with $T_{hit}<\infty$. Let $s_{k_1},s_{k_2}$ be two states and let $\pi_k(\cdot)$ be the policy followed during the $k$-th step. Finally, let $Y_k$ be a geometric random variable with parameter (success probability) $=\frac{1}{2}$. Then
$$P(Y_k > {x}) \geq P\({\frac{v_k(s_{k_1},\pi(s_{k_1}))}{2T_{hit}\max \{ 1, N_k(s_{k_2},\pi(s_{k_2})) \} }} > x \), \forall x \in \mathbb{R}$$
and the bound holds even when conditioned on any random variable ${\frac{v_j(s_{j_1},\pi(s_{j_1}))}{2T_{hit}\max \{ 1, N_j(s_{j_2},\pi(s_{j_2})) \} }}$ for $j < k$.
\end{lemma}
\begin{proof}

By Markov inequality:

\begin{align*} 
P\Big( {\frac{v_k(s_{k_1},\pi_k(s_{k_1}))}{2 T_{hit} \max\{1,N_k(s_{k_2},\pi (s_{k_2}))\}} }\geq 1 \Big) 
& \leq \frac{1}{2} \E \( \frac{v_k(s_{k_1},\pi_k(s_{k_1}))}{T_{hit} \max\{1,N_k(s_{k_2},\pi (s_{k_2}))\}} \) \\
& = \frac{1}{2} \E \( \E \( \frac{v_k(s_{k_1},\pi_k(s_{k_1}))}{T_{hit} \max\{1,N_k(s_{k_2},\pi (s_{k_2}))\}} | N_k(s_{k_2},\pi (s_{k_2})) \) \) \\
& = \frac{1}{2} \E \( \( \frac{\E v_k(s_{k_1},\pi_k(s_{k_1}))}{T_{hit} \max\{1,N_k(s_{k_2},\pi (s_{k_2}))\}} | N_k(s_{k_2},\pi (s_{k_2})) \) \) \\
& \leq \frac{1}{2} \E \( 1 | N_k(s_{k_2},\pi (s_{k_2})) \)  \\
& = \frac{1}{2}
\numberthis \label{terminationprob}
\end{align*}
where lemma \ref{lem:nvisits} was used for the last inequality.
Next we use the above inequality to bound the CDF of ${\frac{v_k(s_{k_1},\pi_k(s_{k_1}))}{2 T_{hit} \max\{1,N_k(s_{k_2},\pi (s_{k_2}))\}} }$ by that of an appropriately defined geometric random variable $\forall x \in \mathbb R$:

\begin{align*} 
P\Big( {\frac{v_k(s_{k_1},\pi_k(s_{k_1}))}{2 T_{hit} \max\{1,N_k(s_{k_2},\pi (s_{k_2}))\}} }\geq x \Big) &  \leq \(\frac{1}{2} \)^{\floor{x}} \\
& = P(Y_k > \floor{x}) \\
& = P(Y_k > {x}) \\
\end{align*}

The first passage requires the explanation below. Subdivide the $k$-th episode into epochs such that each epoch terminates when the agent visits $(s_{k_1},\pi(s_{k_1}))$ roughly $2 T_{hit} \max\{1,N_k(s_{k_2},\pi (s_{k_2}))$ times. During each epoch the agent has at least $\frac{1}{2}$ probability of terminating the episode thanks to equation \ref{terminationprob}. The last equality follows from recognizing the tail probability of a geometric random variable with success probability $\frac{1}{2}$.
Note that this bound hold when we condition on any history experienced by the agent.
\end{proof}

\subsection{Auxiliary Lemmas}
\begin{lemma}
\label{lem:nvisits}
Consider running $\ucrl$ on an MDP with finite maximum mean hitting time $T_{hit}$, and let $s_{k_1},s_{k_2}$ be two states and $\pi_k(\cdot)$ the policy followed during the $k$-th step. 
We have that
$$\E v_k(s_{k_1},\pi(s_{k_1})) \leq T_{hit}\max\{ 1, N_k(s_{k_2},\pi(s_{k_2})) \}.$$
\end{lemma}
\begin{proof}
Subdivide the $k$-th episode into epochs such that each epoch terminates when the agent hits $s_{k_2}$. 
The last epoch occurs when the $k$-th episode terminates.
By assumption, each epoch can be at most $T_{hit}$ long in expectation. It follows that state $s_{k_1}$ can be visited at most 
$T_{hit}$ times during an epoch, in expectation.
Since \ucrl terminates when the number of visits to a state-action pairs doubles, there cannot be more than $\max\{1,N_k(s_{k_2},\pi (s_{k_2}))\}$ epochs within the $k$-th episode. By linearity, $\E v_k(s_{k_1},\pi(s_{k_1})) \leq T_{hit} \max\{1,N_k(s_{k_2},\pi (s_{k_2}))\}$.
\end{proof}

\begin{lemma}
\label{geometric_bound}
Let $\{ Y \}_{i = 1,...,m}$ be a sequence of i.i.d. geometric random variables with parameter (success probability) $\frac{1}{2}$. Then the following statement holds:
$$
P\(\sum_i^m Y_k - m\E Y_k < \epsilon \) < e^{-\frac{\floor{\epsilon}^2}{2(2m+\floor{\epsilon})}}
$$
\end{lemma}
\begin{proof}
Notice that $\sum_i^m Y_k$ can be viewed as a sum of (a random number of) $X_k$ Bernoulli random variables. In particular

\begin{align*}
P\( \sum_i^m Y_k - m\E Y_k > \epsilon \) &  \stackrel{a}{=}  P\( \sum_i^m Y_k > 2m + {\epsilon} \) \\
& \stackrel{b}{=} P\( \sum_i^{\floor{2m + \epsilon}} X_i < m \)  \\
& \stackrel{c}{=} P\( \sum_i^{\floor{2m + \epsilon}} X_i - \frac{1}{2}\floor{2m + \epsilon} < - \frac{1}{2}\floor{\epsilon} \) \\
& \stackrel{d}{<} e^{-\frac{\floor{\epsilon}^2}{2 \(2m+\floor{\epsilon}\)}} \\
\end{align*}
where:
\begin{enumerate}[label=\alph*)]
\item follows because $\E Y_k = 2$
\item the event that the sum of $m$ geometric random variables exceeds $2m+\epsilon$ is the same as the event that $\floor{2m+\epsilon} $ Bernoulli trials give less than $m$ successes
\item is subtracting an identical quantity $\frac{1}{2}\floor{2m+\epsilon}$
\item is Hoeffding inequality
\end{enumerate}
\end{proof}

\end{document}